\documentclass[preprint,12pt,a4paper]{elsarticle}

\usepackage{amsfonts}

\usepackage{graphicx}
\usepackage{subfigure}
\usepackage{amsmath,amsthm,verbatim,amssymb}
\usepackage{xcolor}
\usepackage{lscape}
\usepackage{threeparttable}
\usepackage{float}
\usepackage{mathrsfs}
\usepackage{booktabs}
\usepackage{multirow}
\usepackage{enumerate}
\usepackage{microtype}

\usepackage[T1]{fontenc}
\usepackage{textcomp}

\newtheorem{theorem}{Theorem}[section]
\newtheorem{lemma}{Lemma}[section]
\newtheorem{lemma*}{Lemma}
\newtheorem{remark}{Remark}[section]
\newtheorem{remark*}{Remark}

\newtheorem{definition}{Definition}[section]
\newtheorem{assumption}{Assumption}[section]

\usepackage{hyperref}
\hypersetup{
colorlinks=true,
linkcolor=blue,
anchorcolor=blue,
citecolor=blue}

\def\Ebb{\mathbb{E}}

\def \cA {{\cal A}}
\def \cB {{\cal B}}

\def \cD {{\cal D}}

\def \cH {{\cal H}}

\def \cL {{\cal L}}

\def \cN {{\cal N}}
\def \cO {{\cal O}}

\def \cQ {{\cal Q}}
\def \cR {{\cal R}}
\def \cS {{\cal S}}
\def \cT {{\cal T}}
\def \cU {{\cal U}}

\def \cW {{\cal W}}
\def \cX {{\cal X}}

\def \bD {\mathbb{D}}
\def \bE {\mathbb{E}}

\def \bS {\mathbb{S}}

\newcommand{\norm}[1]{\left\| #1 \right\|}

\journal{Artificial Intelligence}

\begin{document}

\begin{frontmatter}
	
	\title{Deep Weighted Bellman Residual Minimization for $Q^*$ Estimation}
	
	\author[math-ai,hubei]{Lican Kang}
	\ead{kanglican@whu.edu.cn} 
	
	\author[math-ai,math,hubei]{Jerry Zhijian Yang} 
	\ead{zjyang.math@whu.edu.cn}
	
	\author[ai,hubei]{Cheng Yuan\corref{cor1}} 
	
	\ead{yuancheng@whu.edu.cn}
	
	\author[math]{Chen Zhong}
	\ead{zcmath@whu.edu.cn}
	
	\affiliation[math-ai]{
		organization={Institute for Math and AI},
		addressline={Wuhan University}, 
		city={Wuhan},
		postcode={430072}, 
		country={China}
	}  
	\affiliation[math]{
		organization={School of Mathematics and Statistics},
		addressline={Wuhan University}, 
		city={Wuhan},
		postcode={430072}, 
		country={China}
	}
	\affiliation[ai]{
		organization={School of Artificial Intelligence},
		addressline={Wuhan University}, 
		city={Wuhan},
		postcode={430072}, 
		country={China}
	}       
	\affiliation[hubei]{
		organization={Hubei Key Laboratory of Computational Science},
		addressline={Wuhan University}, 
		city={Wuhan},
		postcode={430072}, 
		country={China}
	}
	
	\cortext[cor1]{Corresponding author}
	
	\begin{abstract}
		Off-policy evaluation  is a foundational component of offline reinforcement learning, aiming to assess and optimize policy performance using pre-collected datasets. However, such datasets often suffer from pronounced challenges, including  distribution shift, 
		$Q$-value overestimation, and low sample utilization efficiency. To address these issues, this paper introduces a weighted Bellman residual minimization framework that incorporates density ratio weighting by effectively integrating expert demonstrations with behavioral data. 
		The proposed weighting scheme departs from the conventional completeness assumption commonly imposed in the theoretical analysis of deep reinforcement learning. We establish a sharp convergence rate for density ratio estimation and derive the  convergence rate for the excess risk of resulting deep $Q^*$ estimator.
		Extensive empirical evaluations  demonstrate that, compared to existing methods, our method achieves significant improvements in numerical performance and policy generalization, providing  specific guidance for the rational utilization of expert demonstrations. 
	\end{abstract}
	
	\begin{keyword}
		Off-policy Evaluation
		\sep Density Ratio
		\sep Expert Demonstrations
		\sep Distribution Shift
		\sep Deep Reinforcement Learning.
	\end{keyword}
	
\end{frontmatter}

\section{Introduction}
Deep Reinforcement Learning (DRL) \cite{mousavi2016deep,li2023deep} represents a profound integration of Reinforcement Learning (RL) and Deep Learning (DL). This paradigm leverages the robust representational capabilities of deep neural networks to effectively overcome the inherent limitations of traditional RL in handling high-dimensional and continuous state-action spaces, thereby establishing itself as a central research direction within the field of sequential decision-making \cite{kaelbling1996reinforcement, sutton1998reinforcement}. Driven by breakthrough advancements in deep learning, DRL's applications have expanded across multiple domains, including competitive gaming (e.g., AlphaGo \cite{silver2016mastering}, OpenAI Five \cite{berner2019dota}, and Atari \cite{mnih2015human}), healthcare \cite{yu2021reinforcement, abdellatif2023reinforcement}, and resource management \cite{he2017deep, xiao2021leveraging} and among others. These compelling cases underscore DRL's remarkable ability to make effective decisions within intricate, real-world environments. 

In DRL, value functions and policy functions are commonly represented using deep neural networks (DNNs), enabling the direct application of related deep learning techniques. Notable examples include the deep 
$Q$-network (DQN) \cite{mnih2015human, fan2020theoretical,feng2023over},
double DQN (DDQN) \cite{van2016deep}, proximal policy optimization (PPO) \cite{schulman2017proximal}, soft actor-critic (SAC) \cite{haarnoja2018soft}, deep deterministic policy gradient (DDPG) \cite{tiong2020deep, sumiea2024deep},  minimax squared Bellman optimality error minimization (MSBO),  minimax average Bellman optimality error minimization (MABO) \cite{xie2020Qapproximation}, and deep approximate policy iteration (DAPI) \cite{jiao2025deep}
and among others. The core of DQN lies in the deep fitted $Q$-iteration framework \cite{riedmiller2005neural,antos2007fitted,fan2020theoretical,feng2023over}, 
which iteratively solves a least-squares regression problem to approximate the optimal action-value function. 
Despite these advancements, DRL algorithms still faces substantial practical challenges.
This procedure inherently requires a large and diverse set of trajectory samples to ensure stable learning and accurate value estimation.  
For example,  while DQN   enables
large-scale sample training via 
multi-trajectory data, it suffers from $Q$-value overestimation \cite{van2016deep, fan2020theoretical} and is primarily limited to discrete action spaces, requiring targeted improvements for continuous environments.
In contrast, DDPG performs superiorly in policy learning for continuous control tasks \cite{tiong2020deep, sumiea2024deep} yet remains unable to fully avoid $Q$-value overestimation \cite{tiong2020deep, liu2024evaluation}. Notably, in the off-policy evaluation (OPE) framework, the core challenge lying in addressing data distribution shift \cite{Daume2006domain, torralba2011unbiased, jiang2007instance, carroll2022estimating, tan2016improved}. To tackle this, the MABO algorithm achieves bias correction for distribution shifts through an explicit importance weight mechanism but involves repeated iterations of multiple minimax equation sets, leading to considerable implementation difficulty \cite{xie2020Qapproximation}.

To address the aforementioned challenges, this paper proposes a weighted Bellman residual minimization framework. The method operates in two stages.
In the first stage, expert demonstrations and data generated by the behavior policy are jointly utilized to estimate the density ratio weight, thereby mitigating the discrepancies arising from distributional shifts between datasets.
In the second stage, this obtained weight is incorporated into the Bellman residual minimization 
process to obtain an accurate estimator of the optimal 
action-value function, where this procedure is implemented within DNNs.
Notably, the incorporation of expert demonstrations introduces novel advantages by eliminating the need for the completeness assumption commonly required in theoretical analyses of DRL; see Section \ref{nonasy} for further details.
Building upon this, we leverage an offline policy to integrate the estimated density ratio weight  into the $Q$-value estimation process. This approach reduces dependence on multiple trajectory samples, thereby enhancing both the efficiency and generalization capability of the algorithm in practical settings.
Meanwhile, the use of density ratios offers a principled mechanism to relax stringent data coverage assumptions typical of offline learning, effectively mitigating distributional biases as demonstrated in prior works 
\cite{chen2022offline, huang2022convergence}. Additionally, the utilization of expert demonstrations, which is rich in 
decision-making logic \cite{miao2025reinforcement}, has grown increasingly prevalent in DRL research, with relevant studies validating its capacity to significantly accelerate learning and improve model performance \cite{hester2018deep, chen2020active}. Our main contributions are summarized as follows:
\begin{itemize}
	\item We propose a weighted Bellman residual minimization framework for deep $Q^*$ estimation. Our approach effectively integrates expert demonstrations with behavioral data, allowing explicit modeling of distribution shifts in offline RL without relying on restrictive assumptions on the form of the underlying data distributions.
	\item In theoretical analysis, our method eliminates the need for the completeness assumption, reducing restrictive prerequisites and improving the applicability of the framework in practical scenarios.
	We rigorously establish the convergence rate of the resulting deep estimator, alongside a sharp convergence rate for the associated density ratio estimator.
\end{itemize}

\subsection{Related Work}
Offline RL stands as a pivotal paradigm in RL, focusing on learning from 
pre-collected fixed experience datasets rather than acquiring data through continuous real-time interaction with the environment, as seen in traditional online RL \cite{lange2012batch,chen2019information,xie2020Qapproximation}. Within this domain, fitted $Q$-iteration \cite{riedmiller2005neural,antos2007fitted}
emerges as one of the most representative algorithms. 
By transforming the original RL problem into a series of supervised regression tasks, it learns the optimal 
action-value function 
from a limited set of empirical samples, laying a crucial groundwork for subsequent research \cite{riedmiller2005neural, fu2019diagnosing}. 
DQN,
first proposed by \cite{mnih2015human}, inherits the ideas of fitted $Q$-iteration. Using DNNs to extract complex state features, it handles 
high-dimensional spaces effectively, achieving strong results in tasks like Atari games and enabling end-to-end learning from raw inputs (e.g., game screens) to action-value estimation \cite{bellemare2013arcade, fu2019diagnosing}. Its experience replay mechanism reduces data correlation, stabilizing training \cite{hester2018deep}. However, DQN has obvious drawbacks: it not only exhibits low sample utilization when supporting large-scale sample training through multi-trajectory data but also suffers from $Q$-value overestimation \cite{van2016deep, fan2020theoretical}, and it is mostly limited to discrete action spaces, necessitating targeted extensions for continuous environments. To mitigate the overestimation issue of DQN, DDQN separates action selection and value 
evaluation—using the current network to select actions and the target network to evaluate action 
values—resulting in more accurate value estimates \cite{van2016deep, hester2018deep}. Prioritized DDQN further incorporates a prioritized experience replay mechanism, which measures the priority of experiences based on TD errors, prioritizing the learning of more important transition experiences and improving learning efficiency \cite{hester2018deep}.

In handling continuous action spaces, the Actor-Critic (AC) framework designs a policy network (Actor) to generate an action policy distribution and a value network (Critic) to evaluate values, with the two collaborating and iteratively optimizing. While offering strong flexibility, it also has problems such as high variance, sensitivity to hyperparameters, value overestimation, and high computational resource requirements \cite{konda1999actor,grondman2012survey,haarnoja2018soft}. 
DDPG is built on deterministic policy gradient algorithms, combining the AC framework with DQN, where the actor directly outputs deterministic actions. It addresses traditional algorithms’ continuous action limitations and shows practical value in tasks such as robotic arm control and autonomous driving \cite{hou2017novel}. Nevertheless, it cannot completely avoid $Q$-value overestimation, as the accumulation of biases in the Critic network's evaluation of action values may cause the agent to fall into suboptimal policies, leading to unstable training \cite{tiong2020deep, liu2024evaluation}.

In the OPE framework,  addressing data distribution shift is a core challenge \cite{Daume2006domain, torralba2011unbiased, jiang2007instance, carroll2022estimating, tan2016improved}. A key technique is density ratio estimation with importance sampling, which reweights behavioral policy data to correct shifts, aiding target policy evaluation and unbiased loss construction in learning \cite{sugiyama2012density, chen2022offline, huang2022convergence}. 
It adjusts sample contributions to align with the target policy, enhancing estimate reliability even when policies diverge. On the other hand, minimax error minimization constructs estimators (e.g., importance weights, value functions) to find robust methods for worst-case offline scenarios with distribution shifts and uncertain models. It balances behavioral-target policy differences, aiming for reliable estimation even under unfavorable data conditions \cite{uehara2020minimax}.
MSBO approximates $Q^*$ by minimizing the maximum squared Bellman error, using auxiliary function classes to achieve linear error propagation with improved error bounds and sample complexity \cite{xie2020Qapproximation}. MABO further relaxes function approximation assumptions via explicit importance weighting, directly estimating average Bellman error through linear spans of auxiliary functions, offering stronger adaptability to distribution shifts \cite{xie2020Qapproximation}. 
Recently, \cite{miao2025reinforcement} used the minimax framework to address batch data from heterogeneous time-stationary Markov decision process (MDP). It solves problems such as the difficulty in estimating individual 
$Q$-functions from heterogeneous data, suboptimal policies caused by insufficient data coverage, and the challenge of balancing individual and group differences. However, these methods involve complex iterative computations over minimax equation sets, demanding high computational resources and optimization expertise, which limits their practicality \cite{uehara2020minimax, xie2020Qapproximation}.

Furthermore, demonstration data plays a crucial role in both imitation learning and inverse RL
\cite{subramanian2016exploration, hester2018deep}. In imitation learning, expert demonstrations enable agents to quickly approximate 
high-quality policies, avoiding inefficient random exploration \cite{hussein2017imitation}. In inverse RL, such data helps infer reward functions for complex tasks where direct design is challenging. It guides agents to recognize good behavior and optimize policies accordingly \cite{ng2000algorithms}, thereby better enhancing the value of demonstration data, overcoming scenario limitations, and improving performance in complex real-world tasks.

\subsection{Outlines}
The rest of this paper is structured as follows: Section \ref{bn}  introduces necessary preliminaries. Section \ref{method} presents the proposed method. Section \ref{nonasy} derives the convergence rate of the density ratio estimator and resulting deep estimator. Section \ref{sec:numerical} shows the numerical experiments.
Section \ref{conlusion} concludes this paper.
Proofs for all lemmas and theorem are deferred to the Appendix \ref{append}.

\section{Preliminaries}\label{bn}

\subsection{Notations}
In this part, we present supplementary notations used throughout the paper. For any $m, n \in \mathbb{R}$, we write $m = \mathcal{O}(n)$ if there exists  a  constant $C>0$ such that $m \leq Cn$, and $m = \Omega(n)$ if there exists a constant $C^{\prime}>0$ such that $m \geq C^{\prime}n$.
We use $\mathbb{N}_{0}$ and $\mathbb{N}$
to denote the set of non-negative integers and strictly positive integers, 
respectively. For a multi-index ${s} = (s_{1}, \ldots, s_d) \in \mathbb{N}_{0}^d$, the symbol $\partial^{{s}}$ denotes the partial differential operator: $\partial^{{s}} = \left( \frac{\partial}{\partial x_1} \right)^{s_1} \cdots \left( \frac{\partial}{\partial x_{d}} \right)^{s_{d}},$
and we adopt the convention that $\partial^{{s}}$ acts as the identity operator when ${s} = {0}$. $\|x\|_q=(\sum_{i=1}^d|x_i|^q)^{\frac{1}{q}}$
is the  $q$-norm ($q\in[1,\infty]$) of a vector $x=(x_1,\ldots,x_d)^{\top}\in\mathbb{R}^d$.
We use the notation $\norm{x}_0$ to represent the number of non-zero elements in the vector $x$. For probability measure $\nu$ and measurable function $Q:\mathbb{R}^d\rightarrow \mathbb{R}$,
we write $\|Q\|_{L^{q}(\nu)}^{q}=\Ebb_{x\sim \nu} |Q(x)|^q$.

\subsection{Markov Decision Process}
A discounted MDP \cite{sutton1998reinforcement,agarwal2019reinforcement} is defined by a quintuple $(\cX, \cA, P, \cR, \gamma)$, where $\cX$ is the state space,
$\cA$ is the action space, $P: \cX \times \cA \subseteq \mathbb{R}^d \rightarrow \mathcal{M}(\cX)$ is the transition probability kernel, $\cR(\cdot \mid X, A)$ refers to the distribution of immediate reward $R(X,A)$, and $\gamma \in [0,1)$ is the discount factor. $\mathcal{M}(\cX)$ here denotes the sets of probability measure on $(\cX, \cB(\cX))$,  such that  $P(\cdot \mid X,A)$ is a probability measure on $(\cX, \cB(\cX))$ for each pair $(X,A) \in \cX \times \cA$, which defines the next-state distribution upon taking action $A$ in state $X$,
and $P(D \mid \cdot,\cdot)$ is one measurable function on $\cX \times \cA$ for every $D \in \cB(\cX)$.
Moreover, let
$\pi(\cdot \mid X)$ denote the stochastic policy which is an associated distribution of the action at state $X$. Given one initial distribution $\nu\in \mathcal{M}(\mathcal{X})$, i.e., $X_1 \sim \nu$,
the offline data  $\{Z_i\}_{i=1}^n=\{X_i,A_i,R_i,X_i^{\prime}\}_{i=1}^n$ with $X_i^{\prime}=X_{i+1}$ is generated by
\begin{align*}
	&X_1\sim \nu,~
	A_i\sim \pi(\cdot\mid X_i),~
	R_i\sim \cR(\cdot \mid X_i, A_i),~
	X_{i}^{\prime}\sim P(\cdot \mid X_i, A_i),~ i=1,\ldots,n.
\end{align*}
In this work, we consider continuous state–action spaces and, without loss of generality, assume that $\mathcal{X} \times \mathcal{A} = [0,1]^d$. Furthermore, we assume that the MDP satisfy the independent and identically distributed (i.i.d.) property,  and we denote by $P^{\pi} \in \mathcal{M}(\cX\times \cA)$ the distribution of state-action pair
$(X,A)$, where the superscript $\pi$ explicitly indicates that the action $A$ is generated according to policy $\pi$. 
In this context, we have two offline datasets, one generated by the behavior policy $\pi^b$ and the other by the expert policy $\pi^e$, with the corresponding state–action distributions denoted by $P^{\pi^b}$ and $P^{\pi^e}$, respectively.

Denote the action-value function as
\begin{equation*}
	Q^\pi(X, A) := \mathbb{E} \left[ \sum_{i=1}^{\infty} \gamma^{i-1} R_i  \mid  X_1 = X, A_1 = A, \pi \right].
\end{equation*}
In the general case, we assume that $R(X,A) \in [0,R_{\max}]$ for each pair $(X,A) \in \cX \times \cA$, then $Q^{\pi}$ take values in $\left[0, \frac{R_{\max}}{1-\gamma}\right]$. 
Suppose there exists a policy $\pi^*$ that maximizes $Q^\pi$ such that $Q^* := Q^{\pi^*}.$ The function $Q^*$ satisfies the optimal Bellman equation $Q^* = \mathcal{T}^* Q^*$, 
where the optimal Bellman operator $\mathcal{T}^*$ is given by
\begin{equation*}  
	\mathcal{T}^* Q(X, A) = \mathbb{E}[R(X, A)] + \gamma \mathbb{E}_{X^{\prime} \sim P(\cdot \mid X, A)} \max_{A^{\prime} \in \mathcal{A}} Q(X^{\prime}, A^{\prime}).
\end{equation*}

\subsection{Deep  Neural Networks} \label{section:ReLU-DNNs}
We now introduce DNNs with Rectified Linear Unit (ReLU) activations. A ReLU DNN is defined by the composition
$$
g_{\varphi}(x) = w_{\mathcal{D}} \circ \sigma \circ w_{\mathcal{D}-1} \circ \sigma \circ \cdots \circ \sigma \circ w_1 \circ \sigma \circ w_0(x), \quad x \in \mathbb{R}^d,
$$
where $\sigma(x)=\max\{0,x\}$ denotes the ReLU activation function applied element-wise, 
$\varphi$ represents the set of network parameters, and $\mathcal{D}$ denotes the depth of the network.  
For each $i=0,1,\ldots,\mathcal{D}$, the affine map is
$$
w_i(x) = W_i x + c_i, \quad W_i \in \mathbb{R}^{k_{i+1}\times k_i},\quad c_i \in \mathbb{R}^{k_{i+1}},
$$
with layer dimensions $(k_0,k_1,\ldots,k_{\mathcal{D}})$, where $k_0=d$ corresponds to the input dimension and $k_{\mathcal{D}}=1$ to the output. Thus, the network contains $\mathcal{D}$ hidden layers and $(\mathcal{D}+1)$ layers in total.
Then, the parameters of DNNs $g_{{\varphi}}(\cdot)$ can be written as
$$
\varphi:=\big((W_0,c_0), (W_1,c_1), \ldots, (W_{\mathcal{D}},c_{\mathcal{D}})\big).
$$
We denote the number of non-zero elements as
$$
\|\varphi\|_0 = \sum_{i=0}^{\mathcal{D}} \Big( \|\operatorname{vec}(W_i)\|_0 + \|c_i\|_0 \Big),
$$
and the largest absolute parameter value as
$$
\|\varphi\|_\infty = \max \Big \{\max _{i \in \{0,\ldots,\mathcal{D}\}}
\|\operatorname{vec}(W_{i})\|_{\infty}, \max  _{i \in \{0,\ldots,\mathcal{D}\}} \|c_{i}\|_{\infty} \Big \},
$$
where $\text{vec}(W)$ denotes the column-wise vectorization of $W$. Thus, the width $\mathcal{W}$, the size $\mathcal{S}$, and the weight bound $\mathcal{B}$ are defined as $\mathcal{W}=\max\{k_1,\ldots,k_{\mathcal{D}}\}$, $\mathcal{S}=\|\varphi\|_0,~\|\varphi\|_\infty \leq \mathcal{B},$ respectively.

\begin{definition}[H{\"o}lder class]\label{holder}
	For $\varsigma>0$ with $\varsigma=p+q$, where $p \in \mathbb{N}_{0},~q \in(0,1]$ and $d \in \mathbb{N}$, we denote the H{\"o}lder class  
	$\mathcal{H}^{\varsigma}\left(\mathbb{R}^{d},M\right)$ as
\begin{align*}
	\mathcal{H}^{\varsigma}\left(\mathbb{R}^{d}, M\right):=\Bigg\{ &h: \mathbb{R}^{d} \rightarrow \mathbb{R}, \max _{\|\widetilde{\alpha}\|_{1} \leq p}\left\|\partial^{\widetilde{\alpha}} h\right\|_{\infty} \leq M, \\
	&~~~~ \max _{\|\widetilde{\alpha}\|_{1}=p} \sup _{x \neq y} \frac{\left|\partial^{\widetilde{\alpha}} h(x)-\partial^{\widetilde{\alpha}} h(y)\right|}{\|x-y\|^{q}_{\infty}} \leq M\Bigg\}.
\end{align*}
	For the set $
	[0,1]^d
	\subseteq \mathbb{R}^{d}$, we briefly denote 
    $$\mathcal{H}^{\varsigma}
	:=\left\{h: [0,1]^d \rightarrow \mathbb{R},~ h \in \mathcal{H}^{\varsigma}\left(\mathbb{R}^{d}, M\right)\right\}.
    $$
	
\end{definition}

\begin{definition}[Covering number]
	For $\varepsilon > 0$, the covering number $\mathcal{N}(\mathbf{G}, \varepsilon, \tilde{\rho})$ associated with a semi-metric $\tilde{\rho}$ on the set $\mathbf{G}$ is defined as
	\begin{align*}
		\mathcal{N}(\mathbf{G}, \varepsilon, \tilde{\rho}) = \min_{\kappa} &\Bigg\{  \text{there exist } g_1, \ldots, g_{\kappa} \text{ such that } \\
		&~~~~ \min_{1 \leq j \leq \kappa} \tilde{\rho}(\tilde{g}, g_j) \leq \varepsilon \text{ for all } \tilde{g} \in \mathbf{G} \Bigg\}.
	\end{align*}
\end{definition}

\section{Methodology}\label{method}
In this section, we present our proposed weighted Bellman residual method.
Starting from the optimal Bellman equation, we formulate the estimation of the optimal action-value function $Q^*$ as the minimizer of the following weighted minimization loss function:
\begin{align*}
	\min _{Q 
		\in \cH
	} \mathcal{L}_{w^*}(Q):= \min _{Q 
		\in \cH
	}\left| \mathbb{E}_{(X,A) \sim P^{\pi^b}} \left[ w^*(X,A) \left( Q(X,A) - \mathcal{T}^* Q(X,A) \right) \right] \right|,
\end{align*}  
where $ w^* := \frac{dP^{\pi^*}}{dP^{\pi^b}} $ denotes the density ratio between the 
state-action distributions induced by the optimal policy $\pi^*$ and behavior policy $\pi^b$.
For $Q^* \in \cH$, it directly follows that  
\begin{align*}
	Q^* \in \arg \min _{Q \in \cH} \mathcal{L}_{w^*}(Q).
\end{align*}  
In this work, $ \cH $ is selected as the H\"older class (as defined in Definition \ref{holder}).

In practice, the optimal policy $\pi^*$ is unknown, rendering direct estimation of the density ratio $w^*$  intractable. To overcome this challenge,  we introduce an 
expert policy $\pi^e$  
as an intermediate distribution.  This allows us 
to decompose the weight by the Radon-Nikodym derivative:
$$w^* = \frac{dP^{\pi^*}}{dP^{\pi^e}} \cdot \frac{dP^{\pi^e}}{dP^{\pi^b}}.$$
The expert policy $\pi^e$ is assumed to approximate $\pi^*$ with bounded error such that
\begin{equation} \label{eq:pi-approximate}
	\norm{1 - \frac{dP^{\pi^*}}{dP^{\pi^e}}}_{L_2 (P^{\pi^b})}\leq \zeta,~~~0<\zeta  \ll  1. 
\end{equation}
Here, $\zeta$ quantifies the approximation error between the expert policy and the optimal policy.
This implies  that $w^e:=\frac{dP^{\pi^e}}{dP^{\pi^b}}$ is a close surrogate to $w^*$. 
Accordingly, we  modify loss function by replacing $w^*$ with $w^e$, yielding:
$$
\mathcal{L}_w(Q) := \left| \mathbb{E}_{(X,A) \sim P^{\pi^b}} \left[ w^e(X,A) \left( Q(X,A) - \mathcal{T}^* Q(X,A) \right) \right] \right|.
$$
Since the  density ratio $w^e$ is typically unknown, we use a density ratio estimator $\widehat{w}$ to approximate $w^e$, which in turn acts as an estimator for $w^*$.
A comprehensive description of the density ratio estimation procedure is provided  at the end of this section.
Consequently, we denote the
estimator of $Q^*$ in terms of empirical risk minimization as
\begin{align*}
	\widehat{Q} \in \arg \min _{Q \in \cQ}  
	\widehat{\mathcal{L}}_{\widehat{w}}(Q),
\end{align*}
where
$$
\widehat{\mathcal{L}}_{\widehat{w}}(Q):=
\left|\frac{1}{n}\sum_{i=1}^n
\widehat{w}(X_i,A_i)(Q(X_i,A_i)-Y_i)\right|
$$
is the empirical minimization loss function with $Y_i:=R_i+\gamma \max\limits_{A^{\prime} \in \mathcal{A}} Q(X_i^{\prime}, A^{\prime})$,
and $\cQ $ denotes a function class of ReLU DNNs as defined in Section \ref{section:ReLU-DNNs}. In this evaluation,    the batch data 
$\bD:=\{X_i, A_i, R_i, X^{\prime}_i\}_{i=1}^n $ are induced by the behavior policy $\pi^b$.

\vspace{1em}

\noindent \textbf{Density Ratio Estimation for $\frac{d P^{\pi^e}}{d P^{\pi^b}}$.}
We  formulate density ratio estimator  by using robust 
density-ratio techniques \cite{kanamori2009least,sugiyama2012density}. The robust density-ratio is developed based on Basu’s power divergence \cite{basu1998robust}, which defines a family of divergences indexed by a parameter to balance robustness and efficiency. This framework avoids the need for nonparametric density estimation and associated complications like bandwidth selection, making it computationally feasible.
For our specific estimator, we derive that
\begin{align*}
	w^e=\arg\min_{w}~{\mathcal{R}}(w):= \arg\min_{w} \mathbb{E}_{(X,A) \sim P^{\pi^b} } [ w(X,A)^{\frac{3}{2}} ]-3\mathbb{E}_{X \sim P^{\pi^e} } [w(X,A)^{\frac{1}{2}} ],
\end{align*}
which directly follows from the fact that $w^e$ is the minimizer of 
\begin{align*}
\min_{w}\mathbb{E}_{(X,A) \sim P^{\pi^b} }
&\Big[w(X,A)^{\frac{1}{2}}\Big(w(X,A)-w^e(X,A) \Big) \\
&~~~~ -2w^e(X,A)\Big(w(X,A)^{\frac{1}{2}}-w^e(X,A)^{\frac{1}{2}}\Big)\Big] .
\end{align*}
Moreover, 
we denote $\mathbb{S}^{b}:=\{X_{i}^{b},A_i^{b}\}_{i=1}^{m}$ and $\mathbb{S}^{e}:=\{X_{i}^{e},A_{i}^{e}\}_{i=1}^{m}$  as the state-action samples drawn from  $P^{\pi^b}$ and  $P^{\pi^e}$, 
respectively. We let $\mathbb{S}:= \mathbb{S}^b \cup \mathbb{S}^e$ denote the union of state-action sample sets from both policies, 
and $\mathbb{S}$ is independent of $\mathbb{D}$.
Then, the empirical risk minimizer of $w^e$
can be defined as
\begin{align*}
	\widehat{w} 
	\in\arg\min_{w\in\mathcal{U}}\widehat{\cR}_{
		{\mathbb{S}}}(w) := 
	\frac{1}{m}\sum_{i=1}^{m}
	w(X_i^b,A_i^b)^{\frac{3}{2}}- \frac{3}{m}\sum_{i=1}^{m} w(X_i^e,A_i^e)^{\frac{1}{2}},
\end{align*}
where  $\widehat{\cR}_{\mathbb{S}}$ denotes the empirical loss function and $\mathcal{U}$ also denotes a function class of ReLU DNNs.

\section{Theoretical Analysis}\label{nonasy}
In this section, we present the most critical results of this paper, with a focus on analyzing the convergence rate of the excess risk of deep $Q^*$ estimator.
This convergence rate involves three components: the density ratio estimation error, the subsequent weighted minimization error,
and the expert policy approximation error $\zeta$ between the expert policy $\pi^e$ and the target policy $\pi^*$.
More preciously, we first construct a convergence rate for the density ratio estimator (Lemma \ref{lem:w-err}), then embed it into the weighted minimization framework, and finally account for the expert policy approximation error   $\zeta$, assumed to be bounded, to derive the convergence rate for the
excess risk of the  deep $Q^*$ estimator. 
A detailed proof sketch is given following the main result.
We impose several additional conditions on the density ratio, as outlined in Assumptions \ref{assump-w-bound}, \ref{assump-w-holder}, and  \ref{assump-Q-holder}. Based on these assumptions, we establish the bounds for the excess risk, with relevant results detailed in Theorem \ref{thm:LQ-err}.
\begin{assumption}
	\label{assump-w-bound}
	The density ratio $w^e(x,a) =\frac{p^{\pi^e}(x,a)}{p^{\pi^b}(x,a)} $ has a lower bound $0<\Phi< 1$ and an upper bound $\Psi\geq1$. That is,
	$$
	\Phi := \inf_{(x,a)\in\mathcal{X}\times \cA}w(x,a)>0,~
	\Psi := \sup_{(x,a)\in\mathcal{X}\times\cA}w(x,a)<\infty.
	$$
\end{assumption}

\begin{assumption}
	\label{assump-w-holder}
	The density ratio $w^e$ is H\"older continuous, i.e., $w^e\in\mathcal{H}^{\alpha}$
	for some   $\alpha>0$. 
\end{assumption}

\begin{assumption}
	\label{assump-Q-holder}
	The optimal action-value function $Q^*$ is H\"older continuous, i.e., $Q^* \in \cH^{\beta}$, for some $\beta>0$.
\end{assumption}

\begin{theorem}\label{thm:LQ-err}
	Suppose that 
	the conditions of
	Lemma \ref{lem:w-err}
	and Assumption \ref{assump-Q-holder} hold. The function class $\mathcal{Q}$ is set  with the size $\mathcal{S} = \mathcal{O}\!\left(n^{\frac{d}{d + 2\beta}} \log n\right)$, the depth $\mathcal{D} = \mathcal{O}(\log n)$, and the weights bound $\mathcal{B} = \mathcal{O}\left(n^{\frac{d}{d + 2\beta}}\right)$, then the excess risk satisfies
	\begin{align*}
		\bE_{{\bD,\bS}}[\cL_{w^*}(\widehat Q) - \cL_{w^*}(Q^*)] &\leq \cO\Big( \frac{R_{\max}\Sigma^2 m^{-\frac{\alpha}{d+2\alpha}}(\log m)^\frac{3}{2}}{1-\gamma}  \Big)+\frac{4\zeta\Sigma R_{\max}}{1-\gamma}\\
		&~ +\cO\Big(\frac{R_{\max}\Sigma n^{-\frac{\beta}{d+2\beta}}(\log n)^\frac{3}{2}}{1-\gamma}\Big),
	\end{align*}
	where $\Sigma:=\Psi/\Phi$.
	Moreover, if $m\geq \Omega\left(\Sigma^{\frac{(d+2\alpha)}{\alpha}}n^{\frac{\beta(d+2\alpha)}{\alpha(d+2\beta)}}\right)$
	and we set $\zeta=\cO\Big(m^{-\frac{\alpha}{d+2\alpha}}\Big)$,
	then we obtain  
	\begin{equation}\label{eq:cL-error}
		\bE_{\bD,\bS}[\cL_{w^*}(\widehat Q) - \cL_{w^*}(Q^*)]\leq \cO\Big(\frac{R_{\max}\Sigma n^{-\frac{\beta}{d+2\beta}}(\log n)^\frac{3}{2}}{1-\gamma}\Big).
	\end{equation}
\end{theorem}

\begin{remark}
	Theorem \ref{thm:LQ-err} provides theoretical guarantees for our proposed weighted bellman residual minimization method.
	We can see that 
	the obtained convergence rate consists of three terms.
	In particular, the first term
	$\cO ( \frac{R_{\max}\Sigma^2 m^{-\frac{\alpha}{d+2\alpha}}(\log m)^\frac{3}{2}}{1-\gamma})$ is
	derived from the density ratio estimation error characterized by Lemma \ref{lem:w-err}.
	The second term 
	$\cO(\frac{R_{\max}\Sigma n^{-\frac{\beta}{d+2\beta}}(\log n)^\frac{3}{2}}{1-\gamma})$
	integrates the statistical and approximation errors
	arising from $Q^*$ estimation. 
	The third term $\frac{4\zeta\Sigma R_{\max}}{1-\gamma}$ is based on the error bound 
	$\zeta$ between $\pi^e$ and $\pi^*$. 
	The further details and proof sketches are provided below.
\end{remark}

\begin{remark} 

	In obtaining the statistical error bound, we assume that the data is i.i.d., as introduced in error analysis of DRL (e.g., \cite{jiang2016doubly,fan2020theoretical,uehara2022review}). 
	This assumption can also be extended to the dependent case by considering sequential dynamic nature of MDPs.
	For  dependent data, a typical approach is to adopt a mixing assumption in order to establish the corresponding statistical error bound
	\cite{feng2023over,jiao2025deep}. 
	Theorem \ref{thm:LQ-err} shows that, after ignoring logarithmic orders and other terms, our result achieves a convergence rate of $\cO\big(n^{-\frac{\beta}{d+2\beta}}\big)$, which matches the convergence rate of MABO  \cite{xie2020Qapproximation}.
	%
	Moreover, a key distinction from existing DRL convergence results is that our work dispenses with the stringent completeness assumption. Existing studies commonly rely on this assumption \cite{chen2019information,fan2020theoretical,feng2023over,jiao2025deep}, which requires that for a function class of 
	DNNs $\mathcal{G}$ and  a H\"older class
	$\mathcal{H}$, the condition $\mathcal{T}^*g \in \mathcal{H}$ must hold for each $g \in \mathcal{G}$.
	In contrast, we introduce a novel technique that utilizes expert demonstrations to bridge the gap between the behavior policy and target policy. This approach not only enables us to achieve a 
	convergence rate for the resulting action-value function estimator
	but also effectively relaxes the constraints in practical applications.

\end{remark}

\noindent \textbf{Proof Sketch of  Theorem \ref{thm:LQ-err}.}
To establish the theorem, we first decompose the excess risk into four components, each corresponding to a distinct source of error. We then derive bounds for each component and combine them to obtain the overall result. The error decomposition is given by 
\begin{align*}
	\mathcal{L}_{w^*}(\widehat{Q}) - \mathcal{L}_{w^*}(Q^*) &\leq \underbrace{\mathcal{L}_{w^*}(\widehat{Q}) - \mathcal{L}_{\widehat{w}}(\widehat{Q})}_{\text{I}} + \underbrace{\mathcal{L}_{\widehat{w}}(Q^*) - \mathcal{L}_{w^*}(Q^*)}_{\text{II}} \\
	&+ 2\underbrace{\sup_{Q \in \mathcal{Q}} \Big|\cL_{\widehat w}(Q)-\widehat{\cL}_{\widehat w}(Q)\Big|}_{\text{III}} + \underbrace{\inf_{Q \in \mathcal{Q}} \mathcal{L}_{\widehat{w}}(Q) - \mathcal{L}_{\widehat{w}}(Q^*)}_{\text{IV}}.
\end{align*}

~~\\
\noindent
\textbf{Bound Terms I-II:}
The convergence rate of the density ratio estimator given in Lemma \ref{lem:w-err} provides a bound for these two terms. The validity of this lemma relies on several preliminary assumptions introduced earlier, which ensure analytical rigor. These two terms arise from the interplay of two components: 
the suboptimality gap $\zeta$ of the expert policy and the deep density ratio error between $\hat{w}$ and $w^e$. First, we leverage the Lipschitz continuity of $\cR(\cdot)$, yielding
\begin{align}\label{eq:er}
	\mathcal{L}_{w^*}(\widehat{Q}) - \mathcal{L}_{\widehat{w}}(\widehat{Q}),~\mathcal{L}_{\widehat{w}}(Q^*) - \mathcal{L}_{w^*}(Q^*)\leq \frac{2R_{\max}}{1-\gamma}\norm{\widehat{w}-w^*}_{L_2 (P^{\pi^b})}.
\end{align}
Meanwhile, the bound for $\norm{\widehat{w}-w^*}_{L_2 (P^{\pi^b})}$ can be derived from $\cR(\widehat{w})-\cR(w^{e})$ and $\zeta$. See the following Lemma \ref{lem:w-err}.
\begin{lemma}
	\label{lem:w-err}
	Suppose that Assumptions \ref{assump-w-bound}, \ref{assump-w-holder}
	and Condition \eqref{eq:pi-approximate} hold.
	We set $\mathcal{U}$ as a ReLU DNN class bounded by $\Sigma$ and  structured with the size $\mathcal{S} = \mathcal{O}\!\left(m^{\frac{d}{d + 2\alpha}} \log m\right)$, the depth $\mathcal{D} = \mathcal{O}(\log m)$, and the weights bound $\mathcal{B} = \mathcal{O}\!\left(m^{\frac{d}{d + 2\alpha}}\right)$, then we have 
	\begin{equation}\label{eq:w-error+zeta}
		\bE_\bS \left[\norm{\widehat{w}-w^*}^2_{L_2 (P^{\pi^b})}\right]\leq \cO\Big(\Sigma^4 m^{-\frac{2\alpha}{d+2\alpha}}(\log m)^3\Big)+2\Sigma^2\zeta^2.
	\end{equation}
	Moreover,  by setting $\zeta=\cO\Big(m^{-\frac{\alpha}{d+2\alpha}}\Big)$, we have
	\begin{equation}\label{eq:w-error}
		\bE_\bS \left[\norm{\widehat{w}-w^*}^2_{L_2 (P^{\pi^b})}\right]\leq \cO\Big(\Sigma^4 m^{-\frac{2\alpha}{d+2\alpha}}(\log m)^3\Big). 
	\end{equation}
\end{lemma}

\begin{remark}
	Boundedness Assumption \ref{assump-w-bound} 
	is a technical condition and is widely used in density ratio estimation \cite{Cortes2010learning,feng2024deep} to ensure the stability of the estimation procedure.  
	Assumption \ref{assump-w-holder} imposes a smoothness condition on the target function, widely used in deep nonparametric estimation \cite{bauer2019deep,schmidt2020nonparametric,Jiao2023deep} to enhance the approximation capability of neural networks and to provide theoretical support for convergence analysis.   
\end{remark}

Combining  Lemma \ref{lem:w-err} with Eq. \eqref{eq:er},  
we obtain the bound for Term I, which is given by
$$
\mathcal{L}_{w^*}(\widehat{Q}) - \mathcal{L}_{\widehat{w}}(\widehat{Q}) \leq\cO\Big(\frac{R_{\max}\Sigma^2 m^{-\frac{\alpha}{d+2\alpha}}(\log m)^\frac{3}{2}}{1-\gamma} \Big)+\frac{2\zeta\Sigma R_{\max}}{1-\gamma}.
$$
Similarly, the bound for Term II is
$$
\mathcal{L}_{\widehat{w}}(Q^*) - \mathcal{L}_{w^*}(Q^*)\leq\cO\Big(\frac{R_{\max}\Sigma^2 m^{-\frac{\alpha}{d+2\alpha}}(\log m)^\frac{3}{2}}{1-\gamma} \Big)+\frac{2\zeta\Sigma R_{\max}}{1-\gamma}.
$$



~~\\
\noindent
\textbf{Bound Term III:}
This term is referred to as the statistical error.
We bound it by using the Rademacher complexity technique \cite{wellner2013weak}. The Rademacher complexity
is formally defined as follows. First, we define
\begin{equation}\label{eq:lwhat-Q-Z}
	\ell_{\widehat{w}} (Q;Z):= \widehat{w}(X,A)[Q(X,A)-{\cT^*} Q(X,A)],~Z:=(X,A,R,X^{\prime}),
\end{equation}
and $\bD=\{Z_i\}_{i=1}^n=\{X_i, A_i, R_i, X^{\prime}_i\}_{i=1}^n$ 
(see Section \ref{method}). The empirical Rademacher complexity of $\mathcal{Q}$ is defined by
\begin{equation*}
	\mathfrak{R}_n(\mathcal{Q}\mid \mathbb{D}): =\mathbb{E}_{\tau \mid \bD}\left[\sup_{Q \in \cQ} \Big| 
	\frac{1}{n}\sum_{i = 1}^n\tau_i\ell_{\widehat{w}} (Q;Z_i)
	\Big|  \mid\bD\right],
\end{equation*}
where $\eta>0$ and $\{\tau_i\}_{i=1}^n$ are i.i.d. Rademacher random variables with $P({\tau_i}=1)=P(\tau_i=-1)=\frac{1}{2}$. The Rademacher complexity of $\cQ$ is given by
\begin{equation*}
	\mathfrak{R}_n(\mathcal{Q}): =\mathbb{E}_{\bD}\mathfrak{R}_n(\mathcal{Q}\mid \mathbb{D}) =\mathbb{E}_{\bD,\tau}\left[\sup_{Q \in \cQ} \Big| 
	\frac{1}{n}\sum_{i = 1}^n\tau_i\ell_{\widehat{w}} (Q;Z_i)
	\Big| \right].   
\end{equation*}
By the empirical covering number and Hoeffding's inequality \cite[Theorem D.2]{mohri2018foundations}
(see Lemma \ref{lemmab3}), we can derive the following expression
\begin{equation*}
	\mathfrak{R}_n(\mathcal{Q})\leq \sqrt{\frac{2C^2\log \mathcal{N}(Q, \delta, \|\cdot\|_\infty)}{n}}+ \frac{4C^2}{n}+(1+\gamma)\Sigma\delta,
\end{equation*}
where $C:=\frac{2R_{\max}\Sigma}{1-\gamma}$. In our analysis, we transform the statistical error into the Rademacher complexity for estimation purposes, thereby establishing
\begin{equation*}
	\mathbb{E}_{\bD} \left[\sup_{Q\in \cQ} \Big| \cL_{\widehat w}(Q)-\widehat{\cL}_{\widehat w}(Q)  \Big| \right]\leq 2\mathfrak{R}_n(\mathcal{Q}).
\end{equation*}
From these, we can obtain the statistical error bound, as shown in the following lemma.

\begin{lemma}
	\label{lem:LQ-sta}
	Assume that  the function class $\mathcal{Q}$ is set with the size $\mathcal{S} = \mathcal{O}\left(n^{\frac{d}{d + 2\beta}} \log n\right)$, the depth $\mathcal{D} = \mathcal{O}(\log n)$ and the bound $\mathcal{B} = \mathcal{O}\left(n^{\frac{d}{d + 2\beta}}\right)$, then the statistical error satisfies
	\begin{equation}
		\mathbb{E}_{\bD} \left[\sup_{Q\in \cQ} \Big| \cL_{\widehat w}(Q)-\widehat{\cL}_{\widehat w}(Q)  \Big| \right]\leq \cO\Big(\frac{R_{\max}\Sigma n^{-\frac{\beta}{d+2\beta}}(\log n)^\frac{3}{2}}{1-\gamma}\Big).
	\end{equation}
\end{lemma}

~~\\
\noindent
\textbf{Bound Term IV:} 
This term  denotes  the approximation error. Under Assumption \ref{assump-Q-holder}, this error can be directly bounded by using results of \cite{jiao2025deep}.

Consequently, combining the bounds of these four terms, we can complete the proof of this theorem. Further details are available in Appendix \ref{sec:proof-LQ-sta}.


\section{Numerical Experiments}\label{sec:numerical}
In this section, we present numerical experiments to evaluate our proposed method in two different environments, comparing its performance against several existing approaches. Specifically, we select DQN \cite{fan2020theoretical}, MABO \cite{xie2020Qapproximation}, and DDPG \cite{sumiea2024deep} as  baseline algorithms for comparison.
In Section \ref{subsec:linear}, we focus on a linear environment with a two-dimensional continuous state   and action space. Subsequently, in Section \ref{subsec:DeepMind}, we   evaluate our method on the MountainCarContinuous environment from OpenAI Gym.

\subsection{Linear Environment}\label{subsec:linear}
We begin by describing the mechanism for generating simulated data. Each individual $i \in [N]$ is initialized with a state $\mathbf{X}_0^i \sim \mathcal{N}(0, \mathbf{I}_2)$, where $\mathbf{I}_2$ denotes the 2-dimensional identity matrix. Over the time horizon $0 \leq t \leq T - 1$, each individual selects 2-dimensional continuous action $\mathbf{A}_t^i$ according to a stationary behavioral policy.
The state transition are given by:
$\mathbf{X}_{t+1}^i = \mathbf{F}\mathbf{X}_t^i + \mathbf{G}\mathbf{A}_t^i + \boldsymbol{\xi}_t^i,$
where $\mathbf{X}_t^\prime =(x_t,y_t) \in \mathbb{R}^2$ is the state vector; $\mathbf{A}_t^\prime = (f_x,f_y) \in \mathbb{R}^2$ represents the action vector; and $\boldsymbol{\xi}_t^i \sim \mathcal{N}(0, 0.01\mathbf{I}_2)$ is Gaussian noise. Both $\mathbf{F}$ and $\mathbf{G}$ are undetermined 2nd-order matrices. The immediate rewards at each time step $t$ are given by:
$R_t = \max\{10-(\mathbf{X}_t^\top\mathbf{X}_t + \mathbf{A}_t^\top\mathbf{R}\mathbf{A}_t) + \varepsilon_t,0\},$
where $\mathbf{R}$ is an undetermined second-order matrix that weights the action cost, and $\varepsilon_t \sim \mathcal{N}(0, 0.01)$ represents reward noise. Besides, we set the discount factor as $\gamma=0.9$.
To incorporate population heterogeneity, we define three groups with distinct configurations of $\mathbf{F}, \mathbf{G},$ and $\mathbf{R}$. Specific details can be found in Table \ref{tab:simple-exp_params1}.

\begin{table}[H]
	\centering
	\caption{System Parameters for Different Experiments}
	\label{tab:simple-exp_params1}
	\begin{tabular}{cccc}
		\toprule
		Parameters &Experiment 1 & Experiment 2& Experiment 3 \\
		\midrule
		$\mathbf{F}$ & $\text{diag}(0.5, 0.5)$ & $\text{diag}(0.6, 0.6)$ & $\text{diag}(0.7, 0.7)$ \\
		\midrule
		$\mathbf{G}$ & $\text{diag}(0.5, 0.5)$ & $\text{diag}(0.4, 0.4)$ & $\text{diag}(0.3, 0.3)$ \\
		\midrule
		$\mathbf{R}$ & $\text{diag}(0.25, 0.25)$ & $\text{diag}(0.16, 0.16)$ & $\text{diag}(0.09, 0.09)$ \\
		\bottomrule
	\end{tabular}
\end{table}

For the above environment, we approximate the optimal policy $\pi^{\star}$ using the numerical solution to the Linear Quadratic Regulation (LQR) \cite{agarwal2019reinforcement}, and define the expert policy as:
\begin{equation}\label{eq:require-expert-policy}
	\pi^e = (1 - \tau)\pi^{\star} + \tau \varepsilon_t .
\end{equation}
To obtain expert demonstrations, we set $\tau = 0.05$.

\vspace{1em}
\noindent\textbf{Experiment Results.} We design the experimental setup based on parameters specified in Table \ref{tab:simple-exp_params1} as follows: Each epoch consists of 20 steps, with a total of 10 training epochs while recording cumulative reward curves. Additionally, we perform an evaluation after each training epoch and carry out $10$ additional evaluation epochs upon training completion. By calculating the cumulative discounted rewards, we determine the theoretical maximum achievable rewards for Experiments $1$, $2$, and $3$ to be approximately $87.84$. Consequently, we set reward asymptotes of $88$ in Figures \ref{fig:exp_linear-1}-\ref{fig:exp_linear-3} as references for the upper bounds. Furthermore, Table \ref{tab:simple-exp_params2} provides the mean of the cumulative discounted rewards during the evaluation phase across different methods and experiments, serving as a supplement to the experimental configuration.

From the curve trends in Figures \ref{fig:exp_linear-1}-\ref{fig:exp_linear-3}, we can conclude that: 
\begin{itemize}
	\item Training Reward: As training progresses, the rewards of all methods exhibit an upward trend. In the initial stage, there are slight fluctuations in rewards due to iterative adjustments of the strategies. Subsequently, the reward curves stabilize, entering a phase of steady optimization.
	\item Evaluation Reward: Our method consistently outperforms other algorithms, demonstrating clear advantages in policy generalization. Across all evaluation epochs, it maintains competitive reward levels with smaller fluctuations than some baseline methods, highlighting its strong adaptability to the evaluation environment.
\end{itemize}

From the quantitative data in Table \ref{tab:simple-exp_params2} (Bold values indicate the best performance among the compared algorithms.), our method demonstrates superiority in average reward across the three experiments, outperforming DQN, DDPG, and MABO and approaching the theoretical maximum reward. Nevertheless, algorithms such as DQN, DDPG, and MABO still exhibit competitive performance in certain scenarios.

\begin{table}[H]
	\centering
	\caption{Means of Linear Environment Results}
	\label{tab:simple-exp_params2}
	\begin{tabular}{cccc}
		\toprule
		Algorithm & Experiment 1 & Experiment 2 & Experiment 3 \\
		\midrule
		DQN    & 85.09 & 84.32 & 84.57 \\
		\midrule
		DDPG   & 84.74 & 85.49 & 86.20 \\
		\midrule
		MABO   & 85.02 & 85.64 & 84.64 \\
		\midrule
		Ours   & \textbf{86.65} & \textbf{86.74} & \textbf{87.10} \\
		\bottomrule
	\end{tabular}
\end{table}

\begin{figure}[th]
	\centering
	\includegraphics[width=0.85\textwidth]{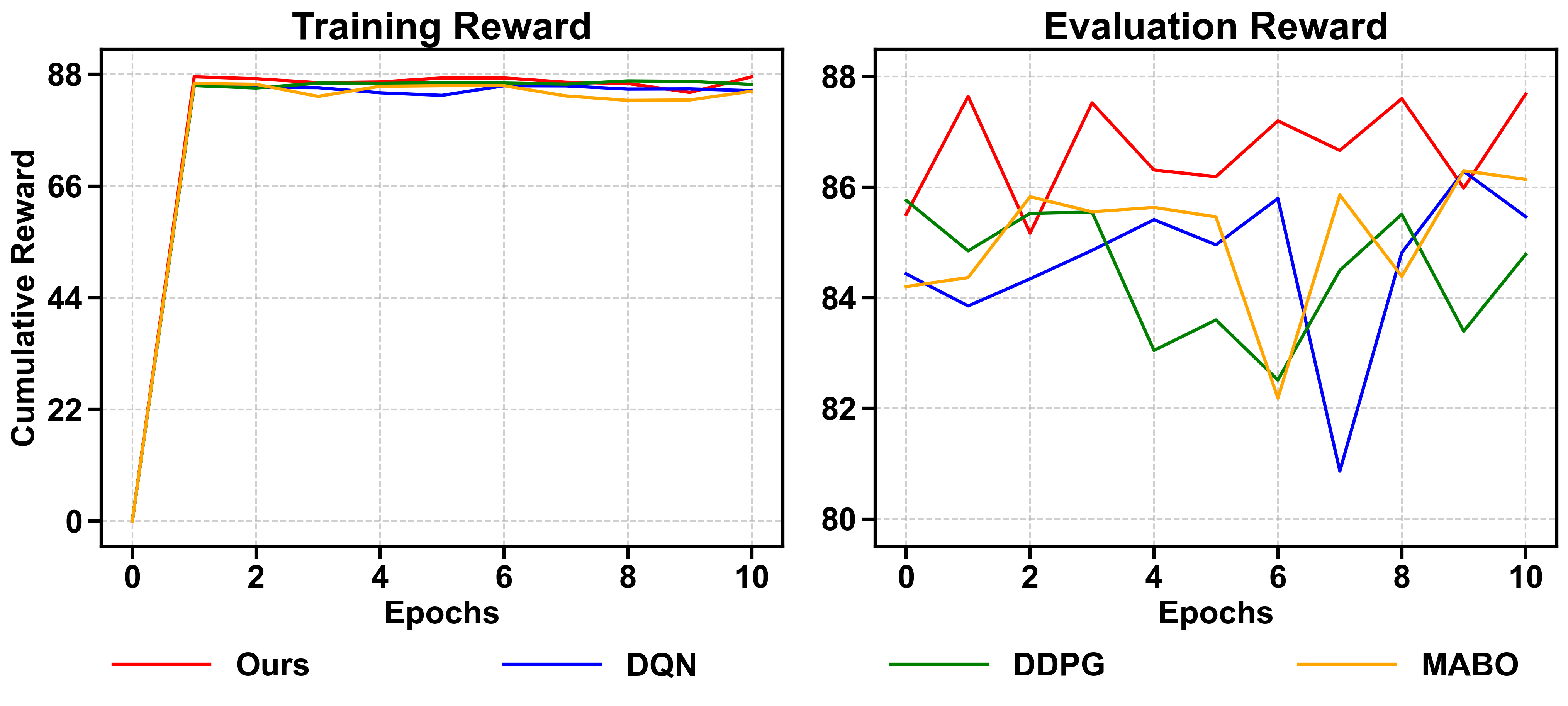} 
	\caption{Numerical results of Experiment 1. The theoretical optimal value is 87.84. Left: Cumulative reward variation with training epochs; Right: Evaluation reward after training completion.} 
	\label{fig:exp_linear-1} 
\end{figure}
\vspace{-0.5em}
\begin{figure}[th]
	\centering
	\includegraphics[width=0.85\textwidth]{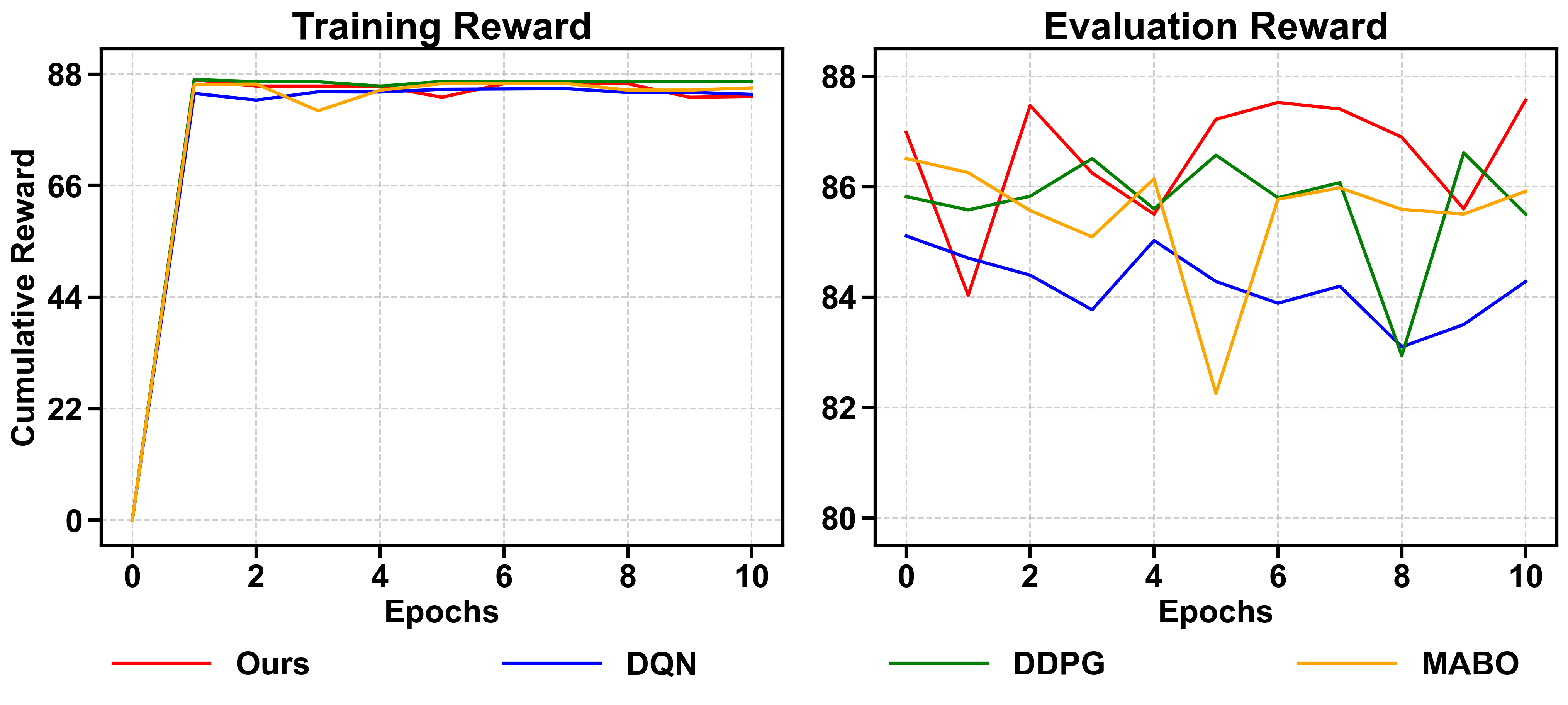} 
	\caption{Numerical results of Experiment 2. The theoretical optimal value is 87.84. Left: Cumulative reward variation with training epochs; Right: Evaluation reward after training completion.} 
	\label{fig:exp_linear-2} 
\end{figure}
\vspace{-0.5em}
\begin{figure}[H]
	\centering
	\includegraphics[width=0.85\textwidth]{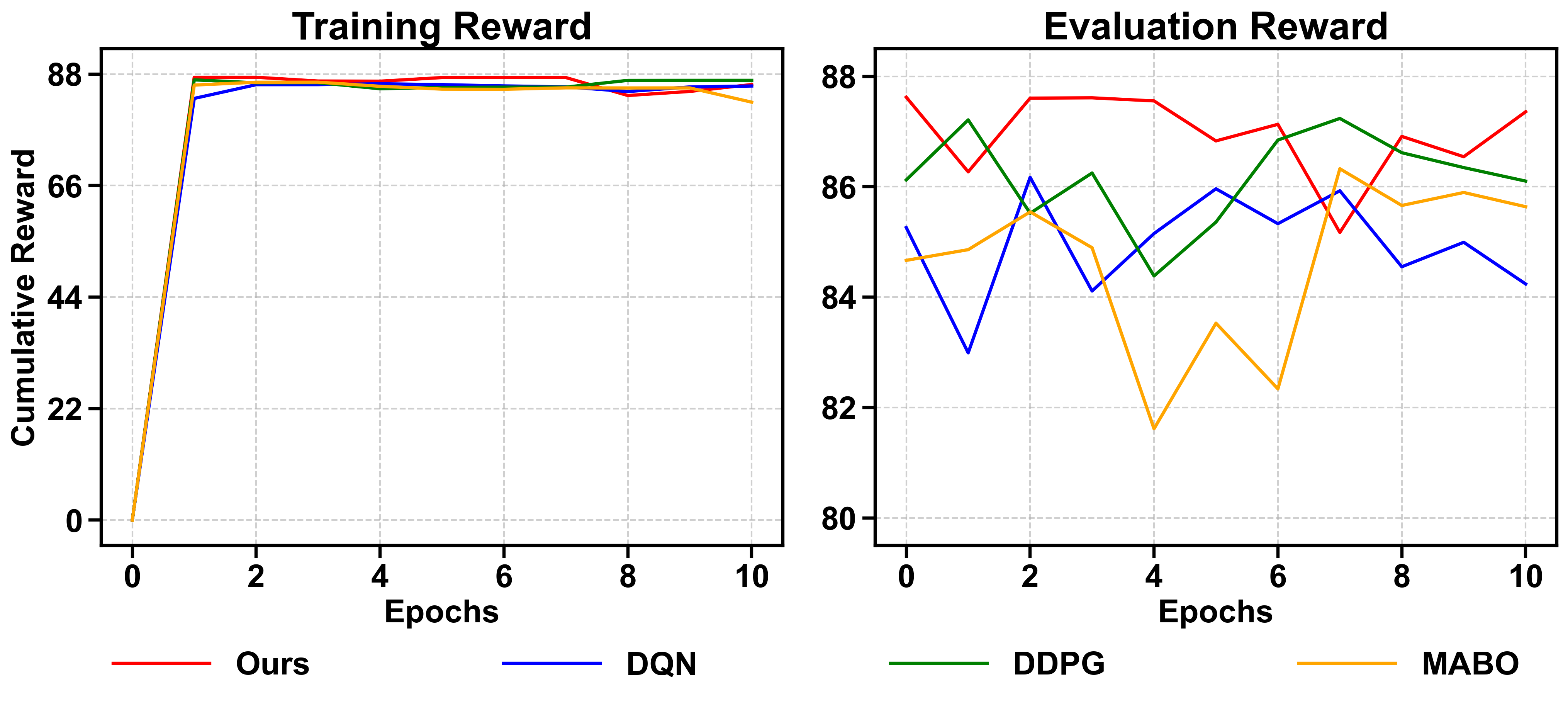} 
	\caption{Numerical results of Experiment 3. The theoretical optimal value is 87.84. Left: Cumulative reward variation with training epochs; Right: Evaluation reward after training completion.} 
	\label{fig:exp_linear-3} 
\end{figure}

\subsection{OpenAI Gym Environment}\label{subsec:DeepMind}

In this experiment, we employ the MountainCarContinuous environment from the Gymnasium library as our testbed. This environment was first appeared in Andrew Moore’s PhD thesis \cite{moore1990efficient} and simulates the dynamic process of a car climbing a steep hillside, with the agent's goal being to precisely control the car's acceleration in order to reach the target at the top.

The state of the environment consists of the car's position and velocity. The agent controls the car's acceleration by outputting a force, and its dynamics follow specific physical rules: the velocity is influenced by both the applied force and gravity, while the position accumulates based on the velocity and is subject to boundary constraints. At each timestep, the agent incurs a penalty of $-0.1 \times \text{action}^2$ to discourage large actions. When the car reaches the target position, a reward of $+100$ is granted. The maximum number of steps per episode is $1000$.

Expert demonstrations and behavioral data are collected through the following processes. For the collection of expert demonstrations, we utilize a closed-form preset policy (\url{https://github.com/openai/gym/wiki/Leaderboard}) that achieves efficient climbing through a position threshold. Following the expert policy construction described in Eq. \eqref{eq:require-expert-policy}, we set $\tau=0.1$ and collect demonstrations per episode, resulting in a total of $20$K data entries. For behavioral data, we gather data per episode, with each episode containing up to $1000$ interaction steps. This data includes both randomly sampled actions and expert demonstrations, totaling approximately $100$K entries.

\noindent\textbf{Experiment Results.} The experiment runs for $500$ training epochs, each consisting of a maximum of $1000$ steps or ending when the car reaches the target position. We conduct an evaluation every $20$ training epochs and record the reward curves. After training, we perform $100$ additional evaluation epochs. For the convenience of observing the experiment, we calculate rewards using undiscounted cumulative rewards. The theoretical maximum reward of this experiment is approximately $100$, and we plot an asymptote of $100$ in Figure \ref{fig:exp_mcarc} as a reference for the upper bound of rewards. To further elaborate on the experimental setup, Table \ref{tab:simple-exp_params3} offers comprehensive quantitative details, with a key focus on the mean of cumulative discounted rewards across various methods in the experiment’s evaluation phase.

From the trends observed in Figure \ref{fig:exp_mcarc}, we  conclude the following:
\begin{itemize}
	\item Training Reward: During the training process, the cumulative rewards of different algorithms exhibit varying patterns. Our method’s reward curve stabilizes after initial fluctuations, steadily achieving higher cumulative rewards as training progresses.
	\item Evaluation Reward: During the evaluation phase, our method generally maintains a competitive reward level. It outperforms DQN, DDPG, and MABO in terms of reward, demonstrating superior policy generalization and better adaptability to the evaluation environment.
\end{itemize}

From the quantitative data in Table \ref{tab:simple-exp_params3} (Bold values indicate the best performance among the compared algorithms.), in the MountainCarContinuous environment, our method achieves the highest mean reward compared to DQN, DDPG, and MABO, demonstrating competitive performance. 

\begin{figure}[th]
	\centering
	\includegraphics[width=0.9\textwidth]{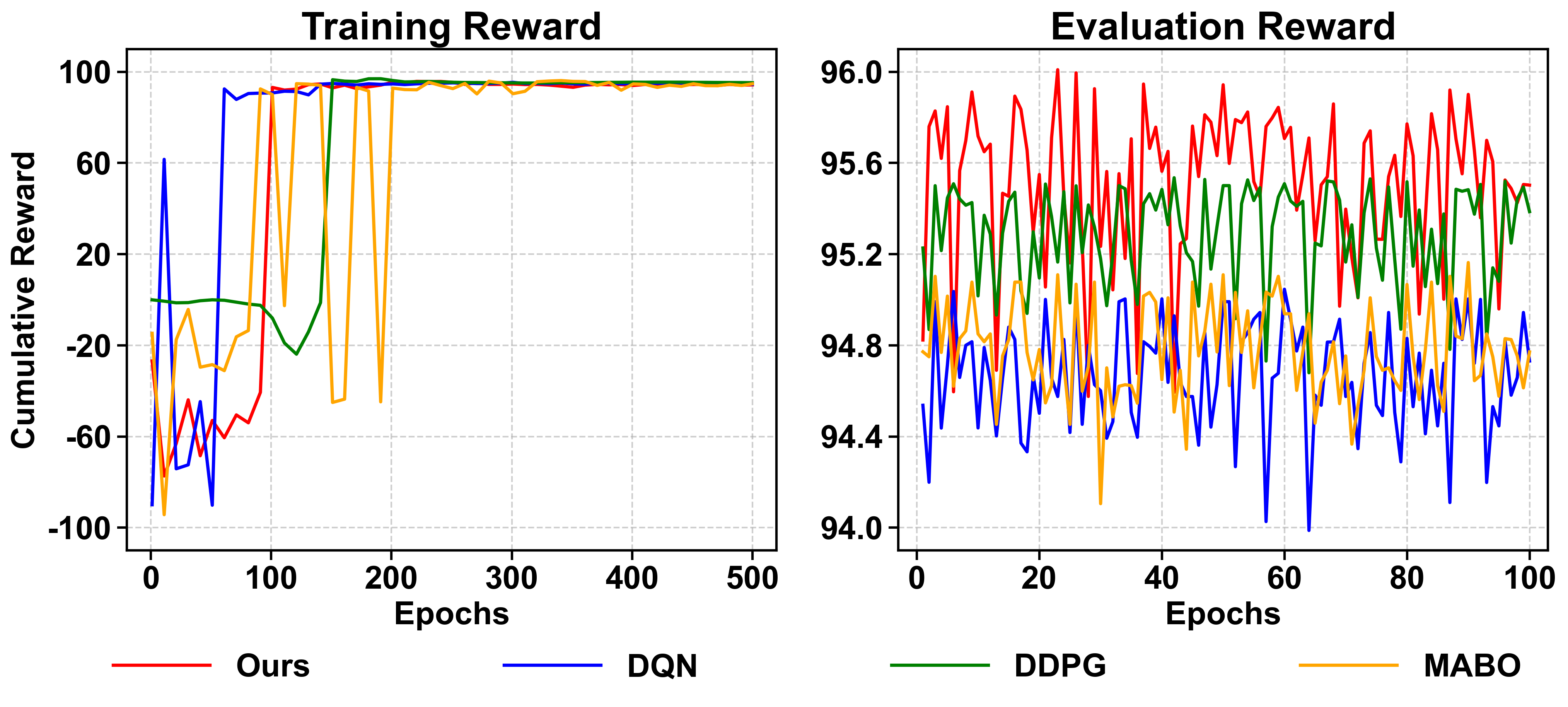} 
	\caption{Numerical analysis results of MountainCarContinuous Experiment. The theoretical optimal value is 100. Left: Cumulative reward variation with training epochs; Right: Evaluation reward after training completion.} 
	\label{fig:exp_mcarc} 
\end{figure}

\vspace{-2em}
\begin{table}[H]
	\centering
	\caption{Means of MountainCarContinuous Environment Results}
	\label{tab:simple-exp_params3}
	\begin{tabular}{ccccc}
		\toprule
		Algorithm & DQN & DDPG & MABO & Ours\\
		\midrule
		Rewards &	94.67  & 95.29 & 94.78 & \textbf{95.51}\\
		\bottomrule
	\end{tabular}
\end{table}

\section{Conclusion}\label{conlusion}
This work introduces a weighted Bellman residual minimization framework designed for deep $Q^* $ estimation. The key feature of our framework is the estimator, which integrates both expert demonstrations and behavioral data to estimate density ratio weights. Notably, our approach dispenses with predefined distributional assumptions, explicitly captures distribution shifts in offline reinforcement learning, and relaxes the conventional completeness requirement.
Theoretically, under the simplification of ignoring logarithmic orders and other terms, we derive a sharp convergence rate for density ratio estimation as $\cO\Big(m^{-\frac{2\alpha}{d+2\alpha}}\Big)$ and establish the convergence rate for the excess risk in $Q^*$ estimation as $\mathcal{O}\left(n^{-\frac{\beta}{d+2\beta}}\right)$.
Experimental results demonstrate that our method provides significant improvements in both numerical performance and policy generalization. Furthermore, it offers valuable practical insights for efficiently utilizing expert demonstrations in real-world applications. In the future, we plan to explore several directions, including expanding the expert paradigm to a broader range of scenarios, and testing its effectiveness in more complex real-world environments, such as robotics control.

\section*{CRediT authorship contribution statement}
\textbf{Lican Kang:} Methodology, Investigation, Conceptualization. \textbf{Jerry Zhijian Yang:} Supervision, Project administration, Funding acquisition. \textbf{Cheng Yuan:} Writing – review \& editing. \textbf{Chen Zhong:} Methodology, Visualization, Software, Data curation.

\section*{Declaration of competing interest}
The authors declare that they have no known competing financial interests or personal relationships that could have appeared to influence the work reported in this paper.  

\section*{Data availability}
No data was used for the research described in the article.

\section*{Acknowledgements}
This work is supported by the National Nature Science Foundation of China (No.12125103, No.U24A2002, No.12301558, No.12501380), the Hubei Natural Science Foundation (No.2025AFA002), and the Fundamental Research Funds for the Central Universities.

\appendix
\section{Proofs of Lemmas and Theorem}\label{append}
In this appendix, we give the detailed proofs for all  lemmas and theorem in this paper.

\subsection{Proof of Lemma \ref{lem:w-err}}
\begin{proof}
	We first verify the strong convexity of 
	$\cR(\cdot)$ at ${w^e}$. For any $ w $,
	\begin{align*}
		&\cR(w)-\cR(w^{e})=\mathbb{E}_{(X,A) \sim P^{\pi^b} }
		\left[
		\begin{aligned}
			&w(X,A)^{\frac{1}{2}}\Big(w(X,A) - w^e(X,A) \Big) \\
			&- 2w^e(X,A)\Big(w(X,A)^{\frac{1}{2}} - w^e(X,A)^{\frac{1}{2}}\Big)
		\end{aligned}
		\right]\\
		&= \mathbb{E}_{(X,A) \sim P^{\pi^b} }
		\left[
		\begin{aligned}
			&\frac{w(X,A)-w^e(X,A)}{w(X,A)^{\frac{1}{2}}+{w^e(X,A)}^{\frac{1}{2}}}\bigg(w(X,A)^{\frac{1}{2}}\Big(w(X,A)^{\frac{1}{2}}\\
			&+{w^e(X,A)}^{\frac{1}{2}}\Big)-2{w^e(X,A)} \bigg) 
		\end{aligned}
		\right]\\
		&= \mathbb{E}_{(X,A) \sim P^{\pi^b} }
		\left[
		\begin{aligned}
			&\frac{w(X,A)-w^e(X,A)}{w(X,A)^{\frac{1}{2}}+{w^e(X,A)}^{\frac{1}{2}}}\bigg( \Big(w(X,A)- w^e(X,A) \Big)\\
			&+w^e(X,A)^{\frac{1}{2}} \Big(w(X,A)^{\frac{1}{2}}-w^e(X,A)^{\frac{1}{2}} \Big) \bigg) 
		\end{aligned}
		\right]\\
		&=\mathbb{E}_{(X,A) \sim P^{\pi^b} }
		\left[\Big(w(X,A)-w^e(X,A)\Big)^2 \frac{w(X,A)^{\frac{1}{2}}+2{w^e(X,A)}^{\frac{1}{2}}}{\Big(w(X,A)^{\frac{1}{2}}+{w^e(X,A)}^{\frac{1}{2}}\Big)^2}     \right].
	\end{align*}
	Under Assumption \ref{assump-w-bound}, we have
	\begin{equation}\label{eq:L-lip}
		C_1\norm{w-{w^e}}^2_{L_2 (P^{\pi^b})} \leq \cR(w)-\cR({w^e})\leq C_2\norm{w-{w^e}}^2_{L_2 (P^{\pi^b})},
	\end{equation}
	where $C_1 := 3\Phi^{\frac{1}{2}}/{4\Psi}, C_2 := 3\Psi^{\frac{1}{2}}/{4\Phi}$.
	Now, we turn to analyze the density ratio error terms arising from the decomposition of $\norm{\widehat{w}-w^*}^2_{L_2 (P^{\pi^b})}$ given in Eq. \eqref{eq:w-decompose}. Consider the decomposition
	\begin{equation}\label{eq:w-decompose}
		\begin{aligned}
			&\norm{\widehat{w}-w^*}^2_{L_2 (P^{\pi^b})}\leq 2\norm{\widehat{w}-w^{e}}^2_{L_2 (P^{\pi^b})}+2\norm{w^{e}-w^*}^2_{L_2 (P^{\pi^b})} \\
			&\leq \frac{2}{C_1}\Big(\cR(\widehat{w})-\cR(w^{e})\Big)+2\norm{{w^e}\Big(1 - \frac{dP^{\pi^*}}{dP^{\pi^e}}\Big)}^2_{L_2 (P^{\pi^b})}\\
			&= \frac{2}{C_1}\Big( \cR(\widehat{w})-2\widehat{\cR}_\bS(\widehat{w})+\cR(w^{e})+2\widehat{\cR}_\bS(\widehat{w})-2\cR(w^{e}) \Big)+2\Psi^2\zeta^2\\
			&\leq \frac{2}{C_1}\Big( \cR(\widehat{w})-2\widehat{\cR}_\bS(\widehat{w})+\cR(w^{e}) \Big)+\frac{4}{C_1}\Big(\cR(\overline{w})-\cR(w^{e})\Big)+2\Psi^2\zeta^2\\
			&\leq \frac{2}{C_1}\Big( \cR(\widehat{w})-2\widehat{\cR}_\bS(\widehat{w})+\cR(w^{e}) \Big)+\frac{4C_2}{C_1} \inf _{w\in \cU}\norm{w-w^{e}}^2_{\infty}+2\Psi^2\zeta^2,\\
		\end{aligned}
	\end{equation}
	where $\overline{w}\in \arg\inf\limits_{w\in\cU}\norm{w-w^{e}}^2_{\cR_2(P^{\pi^b})}$.  This decomposition yields error components: the statistical error $\frac{2}{C_1}\Big( \cR(\widehat{w})-2\widehat{\cR}_\bS(\widehat{w})+\cR(w^{e}) \Big)$ and the approximation error $\frac{4C_2}{C_1} \inf\limits_{w\in \cU}\norm{w-w^{e}}^2_{\infty}$.
	
	To analyze the statistical error, we define
	$$\ell(w;
	\vartheta):=w(x_p,a_p)^{\frac{3}{2}}-3w(x_q,a_q)^{\frac{1}{2}},~\vartheta:=(x_p,~a_p,~x_q,~a_q),
	$$
	and
	$$\tilde{\ell}(w;\vartheta):=\ell(w;\vartheta)-\ell(w^{e};\vartheta).$$
	Then, $\tilde{\ell}(w;\vartheta)$ is Lipschitz continuous over $w$, i.e., for each $\vartheta$, we have
	\begin{align*}
		&|\tilde{\ell}(w_1;\vartheta)-\tilde{\ell}(w_2;\vartheta)|\\
		&=|\big(w_1(x_p,a_p)^{\frac{3}{2}}-3w_1(x_q,a_q)^{\frac{1}{2}}\big)-\big(w_2(x_p,a_p)^{\frac{3}{2}}-3w_1(x_q,a_q)^{\frac{1}{2}}\big)|\\
		&=|\big(w_1(x_p,a_p)^{\frac{3}{2}}-w_2(x_p,a_p)^{\frac{3}{2}}\big)+3\big(w_1(x_q,a_q)^{\frac{1}{2}}-w_1(x_q,a_q)^{\frac{1}{2}}\big)|\\
		&=\Bigg|\big(w_1(x_p,a_p)-w_2(x_p,a_p)\big)\frac{w_1(x_p,a_p)+w_1(x_p,a_p)^{\frac{1}{2}}w_2(x_p,a_p)^{\frac{1}{2}} + w_2(x_p,a_p)}{w_1(x_p,a_p)^{\frac{1}{2}}+w_2(x_p,a_p)^{\frac{1}{2}}}\\
		&~~+3\big(w_1(x_q,a_q)-w_1(x_q,a_q)\big)\frac{1}{w_1(x_q,a_q)^{\frac{1}{2}}+w_1(x_q,a_q)^{\frac{1}{2}}}\Bigg|.
	\end{align*}
	Under Assumption \ref{assump-w-bound}, we obtain
	\begin{equation}\label{eq:l-lip}
		\begin{aligned}
			&|\tilde{\ell}(w_1;\vartheta)-\tilde{\ell}(w_2;\vartheta)|\\
			&\leq \frac{3\Psi}{2\Phi^{\frac{1}{2}}}\big|w_1(x_p,a_p)-w_2(x_p,a_p)\big|+\frac{3}{2\Phi^\frac{1}{2}} \big|w_1(x_q,a_q)-w_2(x_q,a_q)\big|\\
			&=\lambda \Big(\big| w_1(x_p,a_p)-w_2(x_p,a_p)\big|+ \big| w_1(x_q,a_q)-w_2(x_q,a_q) \big| \Big),
		\end{aligned}
	\end{equation}
	where $\lambda:=3\Psi/2\Phi^\frac{1}{2}$. Let $\bS^{\prime} := \{U_1^{\prime},\ldots,U_n^{\prime}\}$ be an i.i.d. ghost sample independent of $\mathbb{S} = \{U_1,\ldots,U_n\}$ with $U_i := (X_i, A_i), i = 1,\ldots,n$. Let 
	$$F(w;U) := \mathbb{E}_{\bS^{\prime}}[\tilde{\ell}(w;U^{\prime}) - 2\tilde{\ell}(w;U)],$$
	then, we have
	\begin{align*}
		&\mathbb{E}_{\bS}\left[\cR(\widehat{w})-2\widehat{\cR}_\bS(\widehat{w})+\cR(w^{e})\right]\\
		&=\mathbb{E}_{\bS} \left[-2\big(\widehat{\cR}_\bS(\widehat{w})-\cR(w^{e})\big)+\big(\cR(\widehat{w})-\cR(w^{e})\big)\right]\\
		&=\mathbb{E}_{\bS}\Big[ \frac{1}{m} \sum_{i=1}^{m} \Big\{ \bE_{\bS^{\prime}}\big(  \ell(\widehat{w};U^{\prime})-\ell(w^{e};U^{\prime}) \big)-2\big( \ell(\widehat{w};U)-\ell(w^{e};U) \big)\Big\} \Big] \\
		&=\mathbb{E}_{\bS}\Big[\frac{1}{m} \sum_{i=1}^{m} \Big\{\bE_{\bS^{\prime}}[\tilde{\ell}(\widehat{w};U^{\prime})-2\tilde{\ell}(\widehat{w};U)]\Big\}  \Big]\\
		&=\mathbb{E}_{\bS}\Big[\frac{1}{m} \sum_{i=1}^{m} F(\widehat{w},U) \Big].
	\end{align*}
	Let $\mathcal{N}(\mathcal{U},\delta,\|\cdot\|_{\infty})$ be the covering number of $\mathcal{U}$ with cover $\mathcal{C} = \{w_1,\ldots,w_{\mathcal{N}}\}$. $\forall w\in\mathcal{U}$, there exists a $\tilde{w}\in\mathcal{C}$ such that
	\begin{align*}
		|\tilde{\ell}(w;\vartheta) - \tilde{\ell}(\tilde{w};\vartheta)| &\leq \lambda\|w- \tilde{w}\|_{\infty} \leq \lambda\delta,\\
		F(w;\vartheta) &\leq F(\tilde{w};\vartheta)+3\lambda\delta.
	\end{align*}
	For simplicity, we denote $\mathcal{N}(\mathcal{U},\delta,\|\cdot\|_{\infty})$ as $ \cN $. By Eq. \eqref{eq:l-lip}, we have
	$$
	|\tilde{\ell}(\tilde{w};U_i)|=|\tilde{\ell}(\tilde{w};U_i)-\tilde{\ell}({w^e};U_i)|\leq  4\lambda \Psi,
	$$
	and so
	$$
	|\tilde{\ell}(\tilde{w};U_i)-\mathbb{E}[\tilde{\ell}(\tilde{w};U_i)]|\leq 8\lambda \Psi=M.
	$$
	By Eq.  \eqref{eq:L-lip}, we get
	$$
	C\norm{\tilde{w}-{w^e}}^2_{L_2 (P^{\pi^b})} \leq \cR(\tilde{w})-\cR({w^e}),~~C=C_1.
	$$
	Denote by $\sigma^2 := \mathrm{Var}(\tilde{\ell}(\tilde{w},U_i))$,  then we obtain
	\begin{align*}
		\sigma^2\leq\mathbb{E}[\tilde{\ell}(\tilde{w};U_i)^2]\leq\lambda^2\norm{\tilde{w}-{w^e}}^2_{L_2 (P^{\pi^b})} \leq\frac{\lambda^2}{C}\mathbb{E} [\tilde{\ell}(\tilde{w};U_i)].
	\end{align*}
	Thus, we have
	$$\mathbb{E}[\tilde{\ell}(\tilde{w};U_i)]\geq\frac{\sigma^2C}{\lambda^2}.$$
	Let $t>3\lambda\delta$ and $v:=\frac{t-3\lambda\delta}{2}+\frac{C\sigma^2}{2\lambda^2}$, then $\frac{\sigma^2}{v}\leq \frac{2\lambda^2}{C}$ and $v\geq \frac{t-3\lambda\delta}{2}$. Therefore, $\forall t>0$, by Bernstein's inequality (see Lemma \ref{lemmab3}), we have 
    
    \begin{align*}
		& P \left[\frac{1}{m}\sum_{i = 1}^{m}F(\widehat{w};U_i)>t\right]\leq P\left[\max_{w\in\mathcal{U}}\frac{1}{m}\sum_{i = 1}^{m}F(w;U_i)>t\right]\\
		&~~\leq  P\left[\max_{\tilde{w}\in\mathcal{C}}\frac{1}{m}\sum_{i = 1}^{m}F(\tilde{w};U_i)>t-3\lambda\delta\right]\\
		&~~ \leq\mathcal{N}\cdot \max_{\tilde{w}\in\mathcal{C}} P\left[\frac{1}{m}\sum_{i = 1}^{m}\mathbb{E}_{\mathbb{S}^{\prime}}[\tilde{\ell}(\tilde{w};U_i^{\prime})]-\frac{2}{m}\sum_{i = 1}^{m}\tilde{\ell}(\tilde{w};U_i)>t-3\lambda\delta\right]\\
		&~~=\mathcal{N}\cdot \max_{\tilde{w}\in\mathcal{C}} P\left[
		\begin{aligned}
			&\mathbb{E}_{\mathbb{S}}\Big[\frac{1}{m}\sum_{i = 1}^{m}\tilde{\ell}(\tilde{w};U_i)\Big]-\frac{1}{m}\sum_{i = 1}^{m}\tilde{\ell}(\tilde{w};U_i)>\\
			&\frac{t-3\lambda\delta}{2}+\mathbb{E}_{\mathbb{S}}\Big[\frac{1}{2n}\sum_{i = 1}^{m}\tilde{\ell}(\tilde{w};U_i)\Big]  
		\end{aligned}
		\right]\\
		&~~ \leq\mathcal{N}\cdot \max_{\tilde{w}\in\mathcal{C}} P\left[\mathbb{E}_{\mathbb{S}}\Big[\frac{1}{m}\sum_{i = 1}^{m}\tilde{\ell}(\tilde{w};U_i)\Big]-\frac{1}{m}\sum_{i = 1}^{m}\tilde{\ell}(\tilde{w};U_i)>\frac{t-3\lambda\delta}{2}+\frac{C\sigma^2}{2\lambda^2}\right]\\
		&~~ =\mathcal{N}\cdot \max_{\tilde{w}\in\mathcal{C}}  P\left[\mathbb{E}_{\mathbb{S}}\Big[\frac{1}{m}\sum_{i = 1}^{m}\tilde{\ell}(\tilde{w};U_i)\Big]-\frac{1}{m}\sum_{i = 1}^{m}\tilde{\ell}(\tilde{w};U_i)>v\right]\\
		&~~ \leq\mathcal{N}\exp \Big(\frac{-mv^2}{2(\sigma^2 + Mv)}\Big)\\
		&~~ =\mathcal{N}\exp\Big(-\frac{m(t-3\lambda\delta)}{\frac{8\lambda ^2}{C}+32\lambda \Psi}\Big).
	\end{align*}
	Setting $\epsilon =3\lambda\delta+\epsilon_0,~\delta=\frac{1}{m}$ and $\epsilon_0=\dfrac{(32\lambda \Psi +\frac{8\lambda ^2}{C})\log\mathcal{N}}{m}$, we have 
	\begin{align*} 
		\mathbb{E}_{\bS}\left[\cR(\widehat{w})-2\widehat{\cR}_\bS(\widehat{w})+\cR({w^e})\right]&\leq \epsilon+\int_{\epsilon}^{\infty}\mathcal{N} \exp\Big(-\frac{m(t-3\lambda\delta)}{\frac{8\lambda ^2}{C}+32\lambda \Psi}\Big) \mathrm{d}t\\
		&\leq \epsilon +\mathcal{N}\exp\left(-\frac{\epsilon_0 m}{32\lambda \Psi +\frac{ 8\lambda^2}{C}}\right)\frac{32\lambda \Psi +\frac{ 8\lambda^2}{C}}{m}\\
		&=\frac{(32\lambda \Psi +\frac{8\lambda ^2}{C})\log\mathcal{N}+3\lambda}{m}\\
		&~~+\mathcal{N}\exp(-\log \mathcal{N})\frac{32\lambda \Psi +\frac{ 8\lambda^2}{C}}{m}\\
		&=\frac{(32\lambda \Psi +\frac{ 8\lambda^2}{C})(\log\mathcal{N}+1)+3\lambda}{m}\\
		&=\frac{(32 C\lambda \Psi +8\lambda^2)(\log\mathcal{N}+1)+3C\lambda}{Cm}.
	\end{align*}
	With Eq. \eqref{eq:w-decompose}, we derive
	\begin{align*}
		\bE_\bS \left[\norm{\widehat{w}-w^*}^2_{L_2 (P^{\pi^b})}\right]
		&\leq \frac{2}{C_1}\mathbb{E}_{\mathbb{S}} \left[\cR(\widehat{w})-2\widehat{\cR}_\bS(\widehat{w})+\cR({w^e})\right]\\
		&~~+\frac{4C_2}{C_1}\inf _{w\in \cU}\norm{w-{w^e}}^2_{\infty}+2\Psi^2\zeta^2\\
		&\leq \frac{2(32 C\lambda \Psi +8\lambda^2)(\log\mathcal{N}+1)+6C\lambda}{C^2m}\\
		&~~+\frac{4C_2}{C_1}\inf _{w\in \cU}\norm{w-{w^e}}^2_{\infty}+2\Psi^2\zeta^2~~(C_1=C)\\
		&\leq \frac{(128\frac{\Psi^3}{\Phi}+64\frac{\Psi^4}{\Phi^2})(\log \cN+1)+12\frac{\Psi^2}{\Phi}}{m}\\
		&~~+4\frac{\Psi^{\frac{3}{2}}}{\Phi^{\frac{3}{2}}} \inf _{w\in \cU}\norm{w-{w^e}}^2_{\infty}+2\Psi^2\zeta^2\\
		&\leq \frac{(128\Sigma^3+64\Sigma^4)(\log \cN+1)+12\Sigma^2}{m}\\
		&~~+4\Sigma^{\frac{3}{2}}\inf _{ w\in \cU}\norm{
			w-{w^e}}^2_{\infty}+2\Sigma^2\zeta^2,
	\end{align*}
	where we define $\Sigma:=\Psi/\Phi$, and the derivation further uses $1/{\Phi}\leq 1/{\Phi^2}.$ Combining this with Lemmas \ref{lem:ApproximationError} and \ref{lem:NBound}, and letting the size $\mathcal{S} = \mathcal{O}\!\left(m^{\frac{d}{d + 2\alpha}} \log m\right)$, the depth $\mathcal{D} = \mathcal{O}(\log m)$, and the weights bound $\mathcal{B} = \mathcal{O}\!\left(m^{\frac{d}{d + 2\alpha}}\right)$, we have
	\begin{align*}
		\bE_\bS \big[\norm{\widehat{w}-w^*}^2_{L_2 (P^{\pi^b})}\big]&\leq \cO\Big(\Sigma^4 m^{-\frac{2\alpha}{d+2\alpha}}(\log m)^3 \Big)+ \cO\Big(\Sigma^3 m^{-\frac{2\alpha}{d+2\alpha}}\Big) +2\Sigma^2\zeta^2\\
		&=\cO\Big(\Sigma^4 m^{-\frac{2\alpha}{d+2\alpha}}(\log m)^3\Big)+2\Sigma^2\zeta^2.
	\end{align*}
	Moreover, by setting $\zeta=\cO\Big(m^{-\frac{\alpha}{d+2\alpha}}\Big)$, we obtain Eq. \eqref{eq:w-error+zeta}.
\end{proof}

\subsection{Proof of Lemma \ref{lem:LQ-sta}} \label{sec:proof-LQ-sta}
\begin{proof}
	Recalling Eq. \eqref{eq:lwhat-Q-Z}, we have 
	\begin{align*}
		|\ell_{\widehat{w}} (Q_1;Z)-\ell_{\widehat{w}} (Q_2;Z)|&\leq \widehat{w}(X,A)|Q_1(X,A)-Q_2(X,A)|\\
		&~~+\widehat{w}(X,A)\gamma|\max_{A^\prime} Q_1(X^\prime,A^\prime)-\max_{A^{\prime}} Q_2(X^\prime,A^\prime)|\\
		& \leq (1+\gamma)\Sigma\max_{(X^{\prime},A^{\prime}) \in \cX \times \cA} |Q_1(X^{\prime},A^{\prime})-Q_2(X^{\prime},A^{\prime})|\\
		&= (1+\gamma)\Sigma\norm{Q_1-Q_2}_{\infty}.
	\end{align*}
	To bound the statistical error, we initiate our analysis by considering the expectation over the distribution $\bD=\{Z_i\}_{i = 1}^n=\{X_i,A_i,R_i,X^{\prime}_{i}\}_{i = 1}^n$, we obtain that
	\begin{align*}
		&\mathbb{E}_{\bD} \left[ \sup_{Q\in \cQ}\Big|  \cL_{\widehat w}(Q)-\widehat{\cL}_{\widehat w}(Q)   \Big| \right]\\
		&=\mathbb{E}_{\bD} \left[ \sup_{Q \in \cQ}
		\Bigg| \Big| \mathbb{E}_{(X,A) \sim P^{\pi^b} }\widehat{w}(X,A)[Q(X,A)-\cT^*Q(X,A)] \Big| \right.\\
		&~~-\left. \Big| \frac{1}{n}\sum_{i = 1}^n \widehat{w}(X_i,A_i)\Big(Q(X_i,A_i)-Y_i\Big) \Big | \Bigg|  \right]\\
		&=\mathbb{E}_{\bD} \left[ \sup_{Q \in \cQ}
		\Bigg| \Big| \mathbb{E}_{(X,A) \sim P^{\pi^b} }[{{\ell}}_{\widehat{w}} (Q;Z)] \Big|
		-\Big| \frac{1}{n}\sum_{i = 1}^n {{\ell}}_{\widehat{w}} (Q;Z_i) \Big | \Bigg|  \right]\\
		&\leq \mathbb{E}_{\bD} \left[ \sup_{Q \in \cQ} \Big|  \mathbb{E}_{(X,A) \sim P^{\pi^b} } [{{\ell}}_{\widehat{w}} (Q;Z)]
		-\frac{1}{n}\sum_{i = 1}^n  {{\ell}}_{\widehat{w}} (Q;Z_i) \Big| \right].
	\end{align*}
	
	Now, we introduce a ghost sample
	$\widetilde{\bD}=\{\widetilde{Z}_i\}_{i = 1}^n=\{ \widetilde{X}_i,  \widetilde{A}_i,  \widetilde{R}_i,  \widetilde{X}^{\prime}_i\}_{i=1}^n$ 
	drawn from $P^{\pi^b}$,  where $\widetilde{\bD}$ is independent of $\bD$. Further, let $\{\tau_i\}_{i=1}^n$ denote a set of Rademacher variables. Substituting the expectation with the empirical mean derived from the ghost sample $\widetilde{\bD}$, we obtain
	\begin{align*}
		&\mathbb{E}_{\bD} \left[ \sup_{Q \in \cQ} \Big|  \mathbb{E}_{(X,A) \sim P^{\pi^b} }
		[{{\ell}}_{\widehat{w}} (Q;Z)]
		-\frac{1}{n}\sum_{i = 1}^n  {{\ell}}_{\widehat{w}} (Q;Z_i) \Big| \right]\\
		&=\mathbb{E}_{\bD} \left[ \sup_{Q \in \cQ} \Big|  \mathbb{E}_{\widetilde{\bD}}[\frac{1}{n}\sum_{i=1}^n 
		{\ell}_{\widehat{w}} (Q;\widetilde{Z}_i)]
		-\frac{1}{n}\sum_{i = 1}^n  {{\ell}}_{\widehat{w}} (Q;Z_i) \Big| \right]\\
		&\leq \mathbb{E}_{\bD,\widetilde{\bD}} \left[ \sup_{Q \in \cQ} \Big| \frac{1}{n}\sum_{i=1}^n \Big( {\ell}_{\widehat{w}} (Q;\widetilde{Z}_i)-{{\ell}}_{\widehat{w}} (Q;Z_i) \Big)\Big| \right]\\
		&=\mathbb{E}_{\bD,\widetilde{\bD},\tau} \left[ \sup_{Q \in \cQ} \Big| \frac{1}{n}\sum_{i=1}^n \tau_i \Big( {\ell}_{\widehat{w}} (Q;\widetilde{Z}_i)-{{\ell}}_{\widehat{w}} (Q;Z_i) \Big)\Big| \right]\\
		& \leq 2\mathbb{E}_{\bD,\tau} \left[ \sup_{Q \in \cQ} \Big| \frac{1}{n}\sum_{i=1}^n \tau_i {\ell}_{\widehat{w}} (Q;Z_i)\Big| \right]\\
		&=2\mathfrak{R}_n(\mathcal{Q}).
	\end{align*}
	Let $\mathcal{N}(\mathcal{Q},~\delta,~\|\cdot\|_{\infty})$ be the covering number of $\mathcal{Q}$ with cover $\mathcal{Q}_{\delta}$, there exists a $Q_{\delta}\in\mathcal{Q}_\delta$ such that $\norm{Q-Q_{\delta}}_{\infty}\leq \delta$, that is
	\begin{align*}
		|\ell_{\widehat{w}} (Q;Z_i)-\ell_{\widehat{w}}(Q_{\delta};Z_i)| \leq (1+y)\Sigma\norm{Q-Q_{\delta}}_{\infty} \leq(1+y)\Sigma\delta.
	\end{align*}
	Thus 
	\begin{align*}
		\left| \frac{1}{n}\sum_{i = 1}^n\tau_i\ell_{\widehat{w}} (Q;Z_i) \right |&\leq \left |\frac{1}{n}\sum_{i = 1}^n\tau_i\ell_{\widehat{w}}(Q_{\delta};Z_i) \right |\\
        &~~ +\frac{1}{n}\sum_{i = 1}^n|\tau_i||\ell_{\widehat{w}} (Q;Z_i)-\ell_{\widehat{w}}(Q_{\delta};Z_i)|\\
		&\leq  \left |\frac{1}{n}\sum_{i = 1}^n\tau_i\ell_{\widehat{w}}(Q_{\delta};Z_i) \right | +(1+\gamma)\Sigma\frac{1}{n}\sum_{i = 1}^n|\tau_i| \norm{Q-Q_\delta}_\infty \\
		&\leq \left |\frac{1}{n}\sum_{i = 1}^n\tau_i\ell_{\widehat{w}}(Q_{\delta};Z_i) \right | +(1+\gamma)\Sigma\delta.
	\end{align*}
	Since $\tau=\{\tau_i\}_{i = 1}^n$ is a sequence of i.i.d. Rademacher variables independent of 
	$\bD=\{Z_i\}_{i=1}^n$, 
	then conditionally on $\bD$ we have
	\begin{align*}
		\mathfrak{R}_n(\mathcal{Q}\mid \bD)=&\mathbb{E}_{\tau \mid \bD}  \left[ \sup_{Q \in \cQ}  \Big| \frac{1}{n}\sum_{i=1}^n \tau_i {\ell}_{\widehat{w}} (Q;Z_i)\Big|\mid\bD\right]\\
		&=\mathbb{E}_{\tau} \left[ \sup_{Q \in \cQ} \Big| \frac{1}{n}\sum_{i=1}^n \tau_i {\ell}_{\widehat{w}} (Q;Z_i)\Big| \right]\\
		&\leq\mathbb{E}_{\tau}\left[ \max_{Q_{\delta}\in\mathcal{Q}_{\delta}} \Big| \frac{1}{n}\sum_{i=1}^n \tau_i {\ell}_{\widehat{w}} (Q;Z_i)\Big| \right] +(1+\gamma)\Sigma\delta,
	\end{align*}
	for any $\eta>0$. 
	Recall that $\{\tau_i\ell_{\widehat{w}}(Q_{\delta};Z_i)\}_{i = 1}^n$ are independent random variables conditioning on $\bD$, then
	$$
	\mathbb{E}_{\tau}[\tau_i\ell_{\widehat{w}}(Q_{\delta};Z_i)] = 0,$$
	and
	$$ -\ell_{\widehat{w}}(Q_{\delta};Z_i)\leq\tau_i\ell_{\widehat{w}}(Q_{\delta};Z_i)\leq\ell_{\widehat{w}}(Q_{\delta};Z_i),\quad i = 1,\ldots,n.
	$$
	By Hoeffding’s inequality (see Lemma \ref{lemmab3}), it yields that for any $Q_{\delta}\in\mathcal{Q}_{\delta}$ and $\xi>0$,
	\begin{align*}
		P_{\tau}\left\{\left| \frac{1}{n}\sum_{i=1}^n \tau_i {\ell}_{\widehat{w}} (Q;Z_i)\right|>\xi\right\}&\leq 2\exp\bigg(-\frac{(n\xi)^2}{2\sum\limits_{i = 1}^n\ell_{\widehat{w}}^2(Q_{\delta};Z_i)}\bigg)\\
		&\leq 2\exp(-\frac{n\xi^2}{2C^2}),
	\end{align*}
	where we used the fact that 
	$$|Q(X,A)| \leq \frac{R_{\max}}{1-\gamma}$$ and $$ \left| \ell_{\widehat{w}}(Q; Z) \right| \leq \frac{2R_{\max}\Sigma}{1-\gamma} =: {{C}}.$$ For simplicity, we denote $\mathcal{N}(\mathcal{Q},\delta,\|\cdot\|_{\infty})$ as $ \cN $. Then
	\begin{align*}
		&\mathbb{E}_{\tau} \left[ \max_{Q_{\delta}\in\mathcal{Q}_{\delta}}\Big| \frac{1}{n}\sum_{i=1}^n \tau_i {\ell}_{\widehat{w}} (Q;Z_i)\Big| \right]\\
		&\leq\int_{0}^{\infty} P_{\tau}\left\{\max_{Q_{\delta}\in\mathcal{Q}_{\delta}}\left( \left| \frac{1}{n}\sum_{i=1}^n \tau_i {\ell}_{\widehat{w}} (Q;Z_i)\right|>\xi\right)\right\}d\xi\\
		&\leq\int_{0}^{\infty} \cN\max_{Q_{\delta}\in\mathcal{Q}_{\delta}} P_{\tau}\left\{\left| \frac{1}{n}\sum_{i=1}^n \tau_i {\ell}_{\widehat{w}} (Q;Z_i)\right| >\xi\right\} d\xi\\
		&\leq A+\int_{A}^{\infty} 2\cN \exp(-\frac{n\xi^2}{2C^2}) d\xi\\
		&\leq  A+ \frac{4\cN C^2}{n} \exp(-\frac{nA^2}{2C^2}).
	\end{align*}
	Setting $A = \sqrt{\frac{2C^2\log \cN}{n}}$, we have
	\begin{gather*}
		\mathbb{E}_{\tau}\left[ \max_{Q_{\delta}\in\mathcal{Q}_{\delta}} \Big| \frac{1}{n}\sum_{i=1}^n \tau_i {\ell}_{\widehat{w}} (Q;Z_i)\Big| \right] \leq \sqrt{\frac{2C^2\log \cN}{n}}+ \frac{4C^2}{n}.
	\end{gather*}
	By applying Lemma \ref{lem:NBound} and choosing $\delta = \frac{1}{n}$ we obtain the following bound
	\begin{align*}
		\mathfrak{R}_n(\mathcal{Q})&\leq \sqrt{\frac{2C^2\log \cN}{n}}+ \frac{4C^2}{n}+(1+\gamma)\Sigma\delta\\
		&\leq \cO \Big( \frac{R_{\max}\Sigma}{1-\gamma}\sqrt{\frac{\cS\cD\log (n\cB\cW\cD)}{n}}\Big).
	\end{align*}
	Considering the size $\mathcal{S} = \mathcal{O}\!\left(n^{\frac{d}{d + 2\beta}} \log n\right)$, the depth $\mathcal{D} = \mathcal{O}(\log n)$, and the weights bound $\mathcal{B} = \mathcal{O}\!\left(n^{\frac{d}{d + 2\beta}}\right)$, we derive the statistical error bound 
	\begin{equation}\label{eq:cL statistical error}
		2\mathbb{E}_{\bD} \left[ \sup_{Q\in \cQ} \Big| \cL_{\widehat w}(Q)-\widehat{\cL}_{\widehat w}(Q)  \Big| \right]
		\leq \cO\Big(\frac{R_{\max}\Sigma n^{-\frac{\beta}{d+2\beta}}(\log n)^3}{1-\gamma}\Big).
	\end{equation}
\end{proof}

\subsection{Proof of Theorem \ref{thm:LQ-err}} \label{sub:poof-of-theorem}
\begin{proof}
	We analyze the excess risk $\cL_{w^*}(\widehat{Q}) - \cL_{w^*}(Q^*)$ using the following decomposition:
	\begin{equation}\label{eq:cL-decompose}
		\begin{aligned}
			\cL_{w^*}(\widehat Q) - \cL_{w^*}(Q^*)&= \cL_{w^*}(\widehat Q)-\cL_{\widehat w}(\widehat Q)+\cL_{\widehat w}(\widehat Q)-\widehat{\cL}_{\widehat w}(\widehat Q)\\
			&~~+\widehat{\cL}_{\widehat w}(\widehat Q)-\widehat{\cL}_{\widehat{w}}(Q)+\widehat{\cL}_{\widehat{w}}(Q)-\cL_{\widehat{w}} (Q)\\
			&~~+\cL_{\widehat{w}} (Q)-\cL_{\widehat{w}}(Q^*)+\cL_{\widehat{w}}(Q^*)-\cL_{w^*}(Q^*)\\
			& \leq \cL_{w^*}(\widehat Q)-\cL_{\widehat w}(\widehat Q)+\cL_{\widehat{w}}(Q^*)-\cL_{w^*}(Q^*)\\
			&~~+\cL_{\widehat w}(\widehat Q)-\widehat{\cL}_{\widehat w}(\widehat Q)+\widehat{\cL}_{\widehat{w}}(Q)-\cL_{\widehat{w}} (Q)\\
			&~~+\cL_{\widehat{w}} (Q)-\cL_{\widehat{w}}(Q^*)\\
			&\leq \cL_{w^*}(\widehat Q)-\cL_{\widehat w}(\widehat Q)+\cL_{\widehat{w}}(Q^*)-\cL_{w^*}(Q^*)\\
			&~~+2\sup_{Q\in \cQ}\Big|\cL_{\widehat w}(Q)-\widehat{\cL}_{\widehat w}(Q)\Big|\\
			&~~+\inf\limits_{Q \in \cQ}\cL_{\widehat{w}}(Q)-\cL_{\widehat{w}}(Q^*).
		\end{aligned}
	\end{equation}
	This decomposition yields four error components: the density ratio errors $\cL_{w^*}(\widehat Q)-\cL_{\widehat w}(\widehat Q)$ and $\cL_{\widehat{w}}(Q^*)-\cL_{w^*}(Q^*)$, the statistical error $2\sup\limits_{Q\in \cQ}\Big|\cL_{\widehat w}(Q)-\widehat{\cL}_{\widehat w}(Q)\Big|$, and the approximation error $\inf\limits_{Q \in \cQ}\cL_{\widehat{w}}(Q)-\cL_{\widehat{w}}(Q^*)$.
	
	Firstly, we bound density ratio error term:
	\begin{align*}
		\cL_{w^*}(\widehat{Q})-\cL_{\widehat{w}}(\widehat{Q})&=\Big| \mathbb{E}_{(X,A) \sim P^{\pi^b} }\Big[{{w^*(X,A)}}\Big({\widehat{Q}(X,A)}-\cT ^* {\widehat{Q}(X,A)}\Big)\Big] \Big|\\
		&~~-\Big| \mathbb{E}_{(X,A) \sim P^{\pi^b} }\Big[{\widehat{w}(X,A)}\Big({\widehat{Q}(X,A)}-\cT ^* {\widehat{Q}(X,A)}\Big)\Big] \Big|\\
		&\leq \Big| \mathbb{E}_{(X,A) \sim P^{\pi^b} }\Big[\Big({w^*(X,A)}-{\widehat{w}(X,A)}\Big)\\
		&~~~~\Big({\widehat{Q}(X,A)}-\cT ^* {\widehat{Q}(X,A)}\Big)\Big] \Big|\\
		&\leq \mathbb{E}_{(X,A) \sim P^{\pi^b} }\left[\Big| \Big({w^*(X,A)}-{\widehat{w}(X,A)}\Big) \right.\\
		&~~~~\left. \Big({\widehat{Q}(X,A)}-\cT ^* {\widehat{Q}(X,A)}\Big) \Big|\right].
	\end{align*}
	Since $|Q(X,A)|\leq \frac{R_{\max}}{1-\gamma}$, then $|\cT^*Q(X,A)|\leq \frac{R_{\max}}{1-\gamma}$. Thus
	\begin{align*}
		\cL_{w^*}(\widehat{Q})-\cL_{\widehat{w}}(\widehat{Q})&\leq \frac{2R_{\max}}{1-\gamma}\norm{{\widehat{w}}-w^*}_{L_2 (P^{\pi^b})}\\
		&\leq \frac{2R_{\max}}{1-\gamma}\norm{{\widehat{w}}-w^e}_{L_2 (P^{\pi^b})}+\frac{2R_{\max}}{1-\gamma}\norm{w^e-w^*}_{L_2 (P^{\pi^b})}\\
		& \leq \frac{2R_{\max}}{1-\gamma}\norm{{\widehat{w}}-w^e}_{L_2 (P^{\pi^b})}+\frac{2\zeta\Sigma R_{\max}}{1-\gamma}.
	\end{align*}
	In the same way, we also obtain
	\begin{align*}
		\cL_{\widehat w}( Q^*)-\cL_{w^*}( Q^*)\leq \frac{2R_{\max}}{1-\gamma}\norm{{\widehat{w}}-w^e}_{L_2 (P^{\pi^b})}+\frac{2\zeta\Sigma R_{\max}}{1-\gamma}.
	\end{align*}
	Taking expectations and by Lemma \ref{lem:w-err}, we have
	\begin{equation}\label{eq:cLQ-w-error}
		\begin{aligned}
			&\bE_{\bD,\bS}\left[\cL_{w^*}(\widehat{Q})-\cL_{\widehat{w}}(\widehat{Q})+\cL_{\widehat w}( Q^*)-\cL_{w^*}( Q^*)\right]\\
			&\leq \frac{4R_{\max}}{1-\gamma}\norm{\widehat{w}-w^e}_{L_2 (P^{\pi^b})}+\frac{4\zeta\Sigma R_{\max}}{1-\gamma}\\
			&\leq \cO\Big(\frac{R_{\max}\Sigma^2 m^{-\frac{\alpha}{d+2\alpha}}(\log m)^\frac{3}{2}}{1-\gamma} \Big)+\frac{4\zeta\Sigma R_{\max}}{1-\gamma}.
		\end{aligned}
	\end{equation}
	We now analyze the approximation error term $\inf\limits_{Q \in \cQ}\cL_{\widehat{w}}(Q) - \cL_{\widehat{w}}(Q^*)$ arising from the decomposition in Eq. \eqref{eq:cL-decompose}.
	To control this error, we bound the error term as follows:
	\begin{align*}
		&\inf_{Q \in \cQ}\cL_{\widehat{w}}(Q)-\cL_{\widehat{w}}(Q^*)=\inf_{Q \in \cQ}\left[\cL_{\widehat{w}}(Q)-\cL_{\widehat{w}}(Q^*)\right]\\
		&= \inf_{Q \in \cQ}\left| \mathbb{E}_{(X,A) \sim P^{\pi^b} }[\widehat{w}(X,A)(Q(X,A)-\cT^*Q(X,A))]\right|\\ 
		&\leq  \Sigma\inf_{Q \in \cQ}\Big| \mathbb{E}_{(X,A) \sim P^{\pi^b} }[Q(X,A)-Q^*(X,A)\\
		&~~+T^*Q^*(X,A)-T^*Q(X,A)] \Big| \\
		&\leq (1+\gamma)\Sigma \Big| \mathbb{E}_{(X,A) \sim P^{\pi^b} }[Q(X,A)-Q^*(X,A)] \Big| \\
		&\leq (1+\gamma)\Sigma \inf_{Q \in \cQ} \norm{Q-Q^*}_\infty.
	\end{align*} 
	
	By Lemma \ref{lem:ApproximationError}, with the size $\mathcal{S} = \mathcal{O}\!\left(n^{\frac{d}{d + 2\beta}} \log n\right)$, the depth $\mathcal{D} = \mathcal{O}(\log n)$, and the weights bound $\mathcal{B} = \mathcal{O}\!\left(n^{\frac{d}{d + 2\beta}}\right)$, we have
	\begin{equation}\label{eq:cL approximation error}
		\begin{aligned}
			\inf_{Q \in \cQ}\cL_{\widehat{w}}(Q)-\cL_{\widehat{w}}(Q^*)&\leq (1+\gamma)\Sigma \inf_{Q \in \cQ} \norm{Q-Q^*}_\infty \\
			&\leq \cO\Big( (1+\gamma)\Sigma n^{-\frac{\beta}{d+2\beta}} \Big). 
		\end{aligned}
	\end{equation}
	Combining this with the density ratio error Eq. \eqref{eq:cLQ-w-error}, statistical error Eq. \eqref{eq:cL statistical error}, and approximation error Eq. \eqref{eq:cL approximation error}, we obtain
	\begin{align*}
		&\bE_{\bD,\bS}\left[\cL_{w^*}(\widehat Q) - \cL_{w^*}(Q^*)\right]\leq \cO\Big( \frac{R_{\max}\Sigma^2 m^{-\frac{\alpha}{d+2\alpha}}(\log m)^\frac{3}{2}}{1-\gamma}  \Big)+\frac{4\zeta\Sigma R_{\max}}{1-\gamma}\\
		&~~+\cO\Big(\frac{R_{\max}\Sigma n^{-\frac{\beta}{d+2\beta}}(\log n)^3}{1-\gamma}\Big)+\cO\Big( (1+\gamma)\Sigma n^{-\frac{\beta}{d+2\beta}} \Big)\\
		&=\cO\Big( \frac{R_{\max}\Sigma^2 m^{-\frac{\alpha}{d+2\alpha}}(\log m)^\frac{3}{2}}{1-\gamma}  \Big)\\
		&~~+\cO\Big(\frac{R_{\max}\Sigma n^{-\frac{\beta}{d+2\beta}}(\log n)^3}{1-\gamma}\Big)+\frac{4\zeta\Sigma R_{\max}}{1-\gamma}.
	\end{align*}
	Furthermore, by assuming $m\geq \Omega\left(\Sigma^{\frac{(d+2\alpha)}{2\alpha}}n^{\frac{\beta(d+2\alpha)}{\alpha(d+\beta)}}\right)$ and $\zeta=\cO\Big(m^{-\frac{\alpha}{d+2\alpha}}\Big)$, the dominant term simplifies Eq. \eqref{eq:cL-error}.
\end{proof}

\section{Auxiliary Results}\label{app:auxiliary}
In this section, we list the lemmas and tools used in the appendix.

\begin{lemma}[Lemma 6.1 of \cite{jiao2025deep}]\label{lem:ApproximationError}
	Assume that $ f \in \mathcal{H}^\varsigma $ with $ \varsigma = s + r $, $ s \in \mathbb{N}_0 $ and $ r \in (0,1] $. For any $ \varepsilon \in (0,1) $, there exists a ReLU DNN function $ \psi $ with depth $ \cD \leq \mathcal{O}\left( \log(1/\varepsilon) \right) $, size $ \mathcal{S} \leq \mathcal{O}\left( \varepsilon^{-d/\varsigma} \log(1/\varepsilon) \right) $, and weight bound $ \mathcal{B} \leq \mathcal{O}\left( \varepsilon^{-d/\varsigma} \right) $ such that
	$$
	\| f - \psi \|_{\infty} \leq \varepsilon.
	$$
\end{lemma}

\begin{lemma}[Lemma 20 of \cite{feng2024deep}]\label{lem:NBound}
	Let $\mathcal{F}$ be the ReLU DNN with width $\mathcal{W}$, depth $\mathcal{D}$, and size $\mathcal{S}$. Assume that the parameters of $\mathcal{F}$ are bounded by a constant $\mathcal{B}>0$, then for each $\delta > 0$,
	$$
	\log\mathcal{N}(\mathcal{F},~\delta,~\|\cdot\|_{\infty})\leq\mathcal{O}(\mathcal{S}\mathcal{D}\log(\mathcal{B}\mathcal{W}\mathcal{D}/\delta)).
	$$
\end{lemma}

\begin{lemma}[Theorem D.2 of \cite{mohri2018foundations}]\label{lemmab3}
	Let \( X_1, \ldots, X_m \) be independent random variables taking values in \( \{-1,1\} \) with \(P(X_i = 1) = p \in [0,1] \) for \( i = 1, \ldots, m \). Then, \( \sum_{i=1}^m X_i \) follows the binomial distribution \( B(m, p) \). We will denote by \( \overline{X} \) the average \( \overline{X} = \frac{1}{m} \sum_{i=1}^m X_i \). Then, the following equality and inequalities hold:
	\begin{align}
		&P( \overline{X} - p > \epsilon ) \leq e^{-2m\epsilon^2} \tag{Hoeffding's inequality} \\ 
		&P( \overline{X} - p > \epsilon ) \leq e^{-\frac{m\epsilon^2}{2\sigma^2 + \frac{2\epsilon}{3}}} \tag{Bernstein's inequality}
	\end{align}
	where \( \sigma^2 = p(1 - p) = \text{Var}~(X_i) \).    
\end{lemma}

\bibliographystyle{plain} 
\bibliography{ref_bib} 

\end{document}